\documentclass[accepted]{arxiv} 

\usepackage[american]{babel}
\usepackage{natbib} 
\hypersetup{
    colorlinks,
    linkcolor={red!50!black},
    citecolor={blue!50!black},
    urlcolor={blue!80!black}
}

\usepackage{microtype}
\usepackage{graphicx}
\usepackage{booktabs} 
\usepackage{amsmath, amssymb}
\usepackage{mathtools}
\usepackage{amsthm}
\usepackage{xcolor}
\usepackage{mathrsfs} 
\usepackage{xfrac} 
\usepackage{dsfont}
\usepackage{mleftright}
\usepackage{subcaption}
\usepackage{multicol}
\usepackage[ruled]{algorithm2e}
\usepackage[noend]{algorithmic}

\newcommand{\bsX}{\boldsymbol{X}}
\newcommand{\bse}{\boldsymbol{e}}
\newcommand{\bszero}{\boldsymbol{0}}
\newcommand{\bsalpha}{\boldsymbol{\alpha}}

\newcommand{\bsc}{\boldsymbol{c}}
\newcommand{\bsZ}{\boldsymbol{Z}}

\newcommand{\bsone}{\boldsymbol{1}}
\newcommand{\bsgamma}{\boldsymbol{\gamma}}
\newcommand{\hbsgamma}{\boldsymbol{\hat\gamma}}
\newcommand{\hbslambda}{\boldsymbol{\hat\lambda}}
\newcommand{\bslambda}{\boldsymbol{\lambda}}
\newcommand{\baralpha}{\bar{\alpha}}

\newcommand{\bsy}{\boldsymbol{y}}

\newcommand{\bsx}{\boldsymbol{x}}
\newcommand{\bsb}{\boldsymbol{b}}

\newcommand{\Prob}{\mathbb{P}}
\newcommand{\excess}{\mathcal{E}}
\newcommand{\lag}{\mathcal{L}}

\newcommand{\Exp}{\mathbb{E}}

\newcommand{\bbR}{\mathbb{R}}
\newcommand{\bbN}{\mathbb{N}}
\newcommand{\risk}{\mathcal{R}}
\newcommand{\class}[1]{\mathcal{#1}}

\newcommand{\eqdef}{\vcentcolon=}

\newcommand{\bszeta}{\boldsymbol{\zeta}}

\newcommand{\parent}[1]{\left( #1 \right)}
\newcommand{\CPs}[1]{\Prob_{\bsX | S = s} \left( #1\right)}
\newcommand{\ECPs}[1]{\hat\Prob_{\bsX | S = s} \left( #1\right)}
\newcommand{\ens}[1]{\left\{ #1\right\}}
\newcommand{\norm}[1]{\left\lVert#1\right\rVert}

\newcommand{\abs}[1]{\left\lvert #1\right\rvert}
\newcommand{\absin}[1]{\lvert #1\rvert}

\newcommand{\ind}[1]{\mathds{1}\left(#1\right)}
\newcommand{\scalar}[2]{\left\langle #1, #2\right\rangle}

\newcommand{\hProb}{\hat{\mathbb{P}}}

\newcommand{\heta}{\hat{\eta}}

\DeclareMathOperator*{\test}{test}
\DeclareMathOperator*{\acc}{acc}
\DeclareMathOperator*{\pos}{pos}

\DeclareMathOperator*{\clf}{clf}

\DeclarePairedDelimiter\ceil{\lceil}{\rceil}

\DeclareMathOperator{\PT}{PT}
\DeclareMathOperator{\NAB}{NAb}
\DeclareMathOperator{\nnZ}{nnz}
\newcommand{\ie}{{\em i.e.,~}}

\newcommand{\eg}{{\em e.g.,}}

\newcommand{\resp}{{\em resp.~}}

\newcommand{\wrt}{{\em w.r.t.~}}
\newcommand{\iid}{{\rm i.i.d.~}}

\newtheorem{theorem}{Theorem}[section]

\newtheorem{proposition}[theorem]{Proposition}
\newtheorem{lemma}[theorem]{Lemma}
\newtheorem{assumption}[theorem]{Assumption}
\newtheorem{remark}[theorem]{Remark}

\newtheorem*{theorem*}{Theorem}
\newtheorem*{proposition*}{Proposition}
\newtheorem*{lemma*}{Lemma}

\usepackage[textwidth=2.0cm, textsize=tiny]{todonotes} 

\title{Classification with abstention but without disparities}

\author[1]{Nicolas Schreuder} 
\author[2]{Evgenii Chzhen}
\affil[1]{%
   CREST, ENSAE, IP Paris
}
\affil[2]{%
    LMO, Université Paris-Saclay, CNRS, Inria
}

\begin{document}
\onecolumn
\maketitle

\begin{abstract}
Classification with abstention has gained a lot of attention in recent years as it allows to incorporate human decision-makers in the process. Yet, abstention can potentially amplify disparities and lead to discriminatory predictions. 
The goal of this work is to build a general purpose classification algorithm, which is able to abstain from prediction, while avoiding disparate impact.
We formalize this problem as risk minimization under fairness and abstention constraints for which we derive the form of the optimal classifier.
Building on this result, we propose a post-processing classification algorithm, which is able to modify any off-the-shelf score-based classifier using only unlabeled sample.
We establish finite sample risk, fairness, and abstention guarantees for the proposed algorithm.
In particular, it is shown that fairness and abstention constraints can be achieved independently from the initial classifier as long as sufficiently many unlabeled data is available.
The risk guarantee is established in terms of the quality of the initial classifier.
Our post-processing scheme reduces to a sparse linear program allowing for an efficient implementation, which we provide.
Finally, we validate our method empirically showing that moderate abstention rates allow to bypass the risk-fairness trade-off.
\end{abstract}

\section{Introduction}
In recent years classification with abstention or with reject option has gained a considerable amount of attention from both statistical and machine learning communities. Probably the earliest appearance of classification with reject option can be found in the works of~\citet{Chow57, Chow70} in the context of information retrieval and an 
initial statistical treatment was given in~\citep{gyorfi_gyorfi_vajda79}.
Much later,~\citet{Herbei_Wegkamp06} provided non-parametric analysis for the problem of binary classification with a fixed rejection cost in the spirit of~\citet{audibert2007fast}. Several extensions followed later, all working with fixed cost of rejection~\citep{YW10,Wegkamp_Yuan11,Bartlett_Wegkamp08}.

Following the conformal prediction literature~\citep[see,  \eg][]{vovk2005algorithmic}, \citet{Lei14} considers a framework where ones wants to minimize the reject rate under a pre-specified accuracy constraint, meanwhile~\citet{denis2020consistency} target its reversed formulation. Both derive finite sample guarantees for plug-in type classification procedures and instanciate their analysis to standard non-parametric class of distributions.
In a similar direction, several practical methods~\citep{grandvalet2008support,nadeem2009accuracy} have been proposed in the machine learning community to address the problem of classification with abstention. Recently, \citet{bousquet2019fast,neu2020fast,puchkin2021exponential} show that abstention can significantly improve regret bounds and convergence rates for the problems of online and batch classification.

Crucially, in our work we view abstention as a mechanism to lighten the burden of fairness constraints and bypass the risk-fairness trade-off
\citep{agarwal2018reductions,pmlr-v81-menon18a,chzhen2020minimax}: one can enjoy the best of both worlds -- a simultaneously fair and accurate classifier -- at the cost of rejection. A majority of observations are still classified in an automatic manner, while the rejected ones can be handled by, \eg~human experts.
Importantly, in our setting, the rejection rate is rigorously controlled by the practitioner depending on the number of available experts.
In addition, since it is illusory to assume that a data-dependent classifier can make error-less and trustworthy decisions, it is desirable to put human experts back in the loop for sensitive tasks. The rejection mechanism partially transfers the burden of optimizing those conflicting quantities to human experts, who can eventually have access to more information to make a better informed decision (\eg~a doctor can ask for extra medical examination for its final diagnosis).

Fairness in binary classification is a very popular topic with various types of algorithmic and statistical contributions~\citep[see,  \eg][]{hardt2016equality,barocas-hardt-narayanan}.
However, abstention framework has not yet received a lot of attention in the context of fair learning. Notable exceptions are work of~\citet{madras2018defer, jones2020selective}.
The latter demonstrates that an imprudent use of abstention might amplify potential disparities already present in the data. In particular, they show that in the framework of prediction without disparate treatment~\citep{zafar2017fairness} the use of the same rejection threshold across sensitive groups might result in a large group-wise risks disparities.
As a potential remedy, our work offers a theoretically grounded way to enforce fairness constraints as well as a desired group-dependent reject rates.
The idea of relying on a reject mechanism to enforce fairness has only been explored once, in \citet{madras2018defer}. The authors introduce ``learning to defer'' framework -- an extension of classification with abstention -- where the cost of rejection is allowed to depend on the prediction of an external decision-maker (\eg~a human expert). The authors argue that by making the automated model aware of the potential biases and weaknesses of the external decision-maker, it can globally optimize for accuracy and fairness. The authors enforce Equalized Odds~\citep{hardt2016equality} through regularization of the risk and thus cannot control explicitly the reject rate, which might potentially lead to a huge external decision-maker costs.
While the authors provide empirical evidences of their claims, theoretical justification of their results remains open.
Our work offers a completely theory-driven way to enforce both fairness and rejection constraints while optimizing for accuracy, leading to a computationally efficient post-processing algorithm.


\textbf{Contributions.} Our work combines and extends previous results in abstention framework with recent results on fair binary classification. Namely, similarly to \citep{denis2020consistency}, we aim at minimizing misclassification risk under a control over \emph{group-wise} reject rates. As we would like to avoid disparate impact, we explicitly add this as a constraint to our framework.
We derive the optimal form of a reject classifier, which minimizes the misclassification risk under the discussed constraints. Our explicit characterization of the optimal reject classifier provides a better understating of the interplay between, on one side, the fairness and rejection constraints and, on the other side, the accuracy.
We propose a data-driven post-processing algorithm which enjoys generic plug-and-play finite sample guarantees. An appealing feature of our post-processing algorithm is that it can be used on top of \emph{any} pre-trained classifier,  thus avoiding the -- potentially high -- cost of re-fitting a classifier from scratch. From numerical perspective, the proposed method reduces to a solution of a sparse linear program, allowing us to leverage efficient LP solvers.
Numerical experiments validate our theoretical result demonstrating that the proposed method successfully enforces fairness and rejection constraints in practice, while achieving a high level of accuracy.

\textbf{Notation.} For each $K \in \bbN$ we denote by $[K]$ the set of the first $K$ positive integers. The standard Euclidean inner product is denoted by $\scalar{\cdot}{\cdot}$. For a real number $a \in \bbR$ we write $(b)_+$ (\resp $(a)_-$) to denote the positive (\resp the negative) part of $a$. For two real numbers $a, b$ we denote by $a \vee b$ (\resp $a \wedge b$) the maximum (\resp the minimum) between the two.
We denote by $\bsone \in \bbR^K$ the vector composed of ones and by $\bse_s \in \bbR^K$ the $s^{\text{th}}$ basis vector of $\bbR^K$.

\section{Problem presentation}
Consider a triplet $(\bsX, S, Y) \sim \Prob$, where $\bsX \in \bbR^d$ is the feature vector, $S \in [K]$ is the sensitive attribute, and $Y \in \{0, 1\}$ is the binary label to be predicted. A classifier is a mapping $g : \bbR^d \times [K] \to \{0, 1, r\}$. That is, any classifier $g$ is able to provide a prediction in $\{0, 1\}$, or to abstain from prediction by outputting $r$.
With any classifier $g$, we associate the following quantities:
\begin{equation}
\label{eq:basics}
\begin{aligned}
    &\risk(g) \eqdef \Prob(Y \neq g(\bsX, S) \mid g(\bsX, S) \neq r)\enspace,\\
    &\NAB_s(g) \eqdef \Prob(g(\bsX, S) \neq r \mid S = s)\enspace,\\
    &\NAB(g) \eqdef \Prob(g(\bsX, S) \neq r)\enspace,\\
    &\PT_s(g) \eqdef \mathbb{P}(g(\bsX, S) = 1 \mid S=s, g(\bsX, S)\neq r)\enspace,\\
    &\PT(g) \eqdef \mathbb{P}(g(\bsX, S) = 1 \mid g(\bsX, S)\neq r))\enspace.
\end{aligned}
\end{equation}
The first one is the risk of a classifier, which measures the probability of incorrect prediction, given that an actual prediction was issued. The second two quantities measure the group-wise and marginal prediction rates. The last two quantities describe the group-wise and marginal rates of positive predictions given that the prediction was made.
Intuitively, a good classifier has low risk $\risk$, high $\NAB_s$, and low disparities between $\PT_s(g)$.

\paragraph{Fairness constraint.} We formalize fairness through the notion of Demographic Parity \citep[see for instance,][]{barocas-hardt-narayanan}. A predictor $g$ is said to satisfy Demographic Parity (or, equivalently, to avoid Disparate Impact) if the distribution of its prediction is independent from the sensitive attribute. Formally, in the standard binary classification framework it means that for any $z \in \{0, 1\}$ and for any $s, s' \in [K]$,
\begin{align*}
    \Prob(g(X,S) = z \mid S=s) = \Prob(g(X,S) = z \mid S=s') \enspace.
\end{align*}
In the setting of classification with abstention, we naturally want to condition on the fact that the classifier issues a prediction, that is, $g(X, S) \neq r$. Using the quantities introduced in Eq.~\eqref{eq:basics}, the latter reduces to
\begin{align*}
    \forall s \in [K],\quad\PT_s(g) = \PT(g)\enspace.
\end{align*}

\paragraph{Penalized version.}
There are various trade-offs that one can consider between the quantities in Eq.~\eqref{eq:basics}. For instance, 
adapting the approach of~\citet{Herbei_Wegkamp06} to the context of fairness,
one can target a prediction which avoids disparate impact and minimizes penalized risk.
Formally, it amounts to solving the following problem:
\begin{equation}
\label{eq:DPWA_pen}
\tag{\textbf{P-DPWA}}
\begin{aligned}
    &\min_{g : \mathbb{R}^d \times [K] \to \{0, 1, r\}} \risk(g) + \sum_{s = 1}^K\lambda_s\NAB_s(g)\\
    &\text{s.t. } \forall s \in [K],\quad
    \PT_s(g) = \PT(g)
\end{aligned}\enspace,
\end{equation}
for some $\lambda_s \geq 0$, $s \in [K]$.
This approach also resembles the one employed by~\cite{madras2018defer}, who additionally penalized for fairness violation instead of directly controlling it.
The main issue with the formulation~\eqref{eq:DPWA_pen} is connected with the choice of the penalization parameters $\lambda_s \geq 0$, $s \in [K]$, which do not have simple and intuitive interpretation.
Indeed, it is impossible to know beforehand which $\lambda_s \geq 0$, $s \in [K]$ will result in a usable reject rate, forcing the practitioner to explore the whole space of the hyperparameters $\lambda_s \geq 0$, $s \in [K]$.
Instead of the above formulation, we consider the problem in which one is able to \emph{explicitly} control the rejection rate. In particular, such an approach allows us to develop a \emph{parameter-free} post-processing method.
\paragraph{Explicit control of reject.}
Given $\bsalpha = (\alpha_1, \ldots, \alpha_K)^\top \in [0, 1]^K$, our goal is to find a solution of the following problem
\begin{equation}
\label{eq:DPWA}
\tag{\textbf{DPWA}}
\begin{aligned}
    &\min_{g : \mathbb{R}^d \times [K] \to \{0, 1, r\}} \risk(g)\\
    &\text{s.t. }, \forall s \in [K],
    \begin{cases}
    \NAB_s(g) = \alpha_s\\
    \PT_s(g) = \PT(g) \,
    \end{cases}
\end{aligned}\enspace.
\end{equation}
It will be shown later that, under a mild assumption on the distribution of the conditional expectation $\Exp[Y \mid \bsX, S]$, the above problem admits a global minimizer written in the form of group-wise thresholding.

The first constraint in~\eqref{eq:DPWA} specifies the abstention level accepted for each class while the second constraint, as before, demands the classifier $g$ to avoid disparate impact.
Notably, in this formulation, the parameter vector $\bsalpha \in [0, 1]^K$ has a simple and intuitive interpretation -- it allows to fix precisely \emph{different} levels of rejects for different groups. This, for instance, can be beneficial, if $g(\bsx, s) = r$ is followed by the intervention of a human decision-maker, who replaces the classifier. One can force a higher rejection rate (\ie a higher rate of human intervention) for disadvantaged groups by lowering the corresponding $\alpha_s \in [0, 1]$. Crucially, we implicitly assume that the practitioner is able to treat unclassified instances in an accurate and fair manner. While this assumption is void for the theoretical contributions of our paper, we \emph{warn} the practitioner that it must not be overlooked once our method is deployed in real world.

This formulation allows to bypass the usual trade-off between fairness and accuracy at the price of rejection. Indeed, note that a classifier that solves~\eqref{eq:DPWA} is fair for any parameters $(\alpha_s)_{s \in [K]}$. At the same time, setting $\alpha_1 = \ldots = \alpha_K = \tilde\alpha$ for some $\tilde\alpha \in (0, 1]$, one can observe that by varying $\tilde\alpha$ we can recover the accuracy of a classifier without constraints while still satisfying Demographic Parity. This will be later empirically confirmed in Section~\ref{sec:exps}. We again emphasize that the accuracy gain comes at a price of a possible reject region, which, depending on the application at hand might or might not constitute a reasonable price.

\section{Optimal classifier}

Our first theoretical contribution is the derivation of a classification strategy $g^*$, which is a solution of~\eqref{eq:DPWA}.
We define the conditional expectation of the label $Y$ knowing $(\bsX, S)$ as  
\begin{align*}
    \eta(\bsX, S) = \Exp[Y \mid \bsX,S] \enspace.
\end{align*}
It is known that the Bayes optimal rule for the problem of binary classification with misclassification risk is given by the point-wise thresholding of $\eta(\bsX, S)$ on the level $1/2$~\citep{devroye2013probabilistic}. In our case the classifier does not correspond to the Bayes decision. Instead, it is a solution of a constrained optimization problem with constraints that depend on the unknown data distribution $\Prob$. In several frameworks, which are also formulated  as risk minimization under distribution dependent constraints, it is possible to obtain a closed form expression of a minimizer under fairly mild assumptions. In particular, it is the case for the classification with reject option~\citep{Chow70,Lei14,denis2020consistency} as well as classification under various fairness constraints~\citep{hardt2016equality,chzhen2019leveraging,del2020review}.
Similarly to the above contributions, we will make a mild assumption on the behaviour of $\eta(\bsX, S)$, which is, for instance, naturally satisfied whenever $\eta(\bsX, S)$ admits a density \wrt the Lebesgue measure.
\begin{assumption}
    \label{as:continuous_eta}
    The random variables $(\eta(\bsX, S) \mid S = s)$ are non-atomic for all $s \in [K]$.
\end{assumption}
One can actually get rid of this assumption, as explained in~\citet{Lei14}, by switching from deterministic classification strategies, which are valued in $\{0, 1, r\}$, to randomized classifiers, which output a distribution over $\{0, 1, r\}$.

To present the main result of this section, we introduce the notations $p_s := \Prob(S = s)$, $\baralpha \coloneqq \sum_{s \in [K]} p_s \alpha_s$ and we define the following function
\begin{align*}
    G(\bsx, s, \bslambda, \bsgamma) =
    &\left\lvert  \frac{p_s}{2\baralpha}(1-2\eta(\bsx,s)-\scalar{\bsgamma}{\bsone}) + \frac{\scalar{\bsgamma}{\bse_s}}{2\alpha_s} \right\rvert 
    - \frac{p_s}{2\baralpha}(1-\scalar{\bsgamma}{\bsone}) - \scalar{\bslambda}{\bse_s} - \frac{\scalar{\bsgamma}{\bse_s}}{2\alpha_s}\enspace,
\end{align*}
which plays a key role in the derivation of an optimal classifier for the problem \eqref{eq:DPWA}.
We now state the first result of this work, which provides a form of $g^*$ -- solution for \eqref{eq:DPWA}.
\begin{theorem}
\label{thm:optimal}
Under Assumption~\ref{as:continuous_eta}, an optimal classifier for the problem \eqref{eq:DPWA} is given for all $(\bsx, s) \in \bbR^d \times [K]$ by
\begin{align*}
    g^*(\bsx,s) = \begin{cases}
       r &\text{if } G(\bsx, s, \bslambda^*, \bsgamma^*) \leq 0 \\
       \mathds{1}\left(\eta(\bsx,s)  \geq  \frac{1}{2} + c_{\bsgamma^*, s}\right) &\text{otherwise}
    \end{cases}\enspace,
\end{align*}
where $(\bslambda^*, \bsgamma^*)$ are solutions of
\begin{align*}
    \min_{(\bslambda, \bsgamma)}\ens{\scalar{\bslambda}{\bsalpha} + \sum_{s=1}^K  \Exp_{\bsX|S=s}[(G(\bsX, S, \bslambda, \bsgamma))_+]}\enspace,
\end{align*}
and $c_{\bsgamma^*, s} \eqdef \frac{1}{2}\big(\frac{\baralpha\gamma_s^*}{\alpha_sp_s} - \scalar{\bsone}{\bsgamma^*}\big)$.
\end{theorem}
Let us mention that unlike other similar results described above, the main difficulty in the proof of Theorem~\ref{thm:optimal} lies in the fact the misclassification risk in our case involves conditioning on the event which itself depends on the classifier that we want to find. 
Theorem~\ref{thm:optimal} is instructive and allows to develop an intuition which is similar to that of the original rule derived by~\cite{Chow57,Chow70}.
To be more precise, denoting by 
\begin{align*}
    t_{\bsgamma^*, s} \coloneqq (1-\scalar{\bsgamma}{1}) + \frac{\baralpha \scalar{\bsgamma}{\bse_s}}{p_s \alpha_s}\enspace,
\end{align*}
the reject region is expressed as a strip around $t_{\bsgamma, s}$:
\begin{align*}
    \lvert \eta(\bsx, s) - t_{\bsgamma, s}  \rvert \leq  t_{\bsgamma, s} + \frac{\baralpha \bslambda_s}{p_s}\enspace.
\end{align*}
We highlight that the center as well as the size of this strip is group-dependent.
Interestingly, the position of the strip only depends on the Lagrange multiplier controlling for the fairness constraint, while its width is determined by both constraints.

\section{Empirical method}
The form of the optimal classifier suggests to develop a post-processing algorithm, which receives an estimator $\heta(\bsx, s)$ of $\eta(\bsx, s)$ and an additional \emph{unlabeled} set of samples to estimate $(\bsgamma^*, \bslambda^*)$.
Indeed, observe that the optimal classifier $g^*$ is know up to the quantities $\eta(\bsx, s), \bsgamma^*, \bslambda^*$.
\begin{remark}
For simplicity of exposition we assume that the marginal distribution of $S$ is known, that is, we have access to $p_s := \Prob(S = s)$. Note that $S$ follows multinomial distribution, and, in practice, we can estimate these probabilities by their empirical counterparts, which is the direction that we take in our experimental section. Our proofs generalize straightforwardly for the case of unknown $p_s$, but such modification results in additional, unnecessary, complications.
\end{remark}

We denote by $\hat\eta(\bsX,S)$ any off-the-shelf estimator of $\eta(\bsX,S)$.
For instance, one can take k-NN~\citep{stone1977consistent,devroye2013probabilistic}, locally polynomial estimator~\citep{korostelev2012minimax}, logistic regression~\citep{buhlmann2011statistics}, random forest~\citep{breiman2001random,biau2016random,mourtada2020minimax} to name a few.
Our theoretical guarantees on the misclassification risk will explicitly depend on the quality of this off-the-shelf estimator, hence it is advisable to use those methods which are supported by statistical guarantees.
Yet, our algorithm remains valid even for inconsistent estimators $\heta$ in the sense that the resulting classifier after post-processing will (nearly) satisfy the prescribed constraints independently from $\heta$.

\begin{remark}
In what follows we assume that the estimator $\heta(X, S)$ is independent from the \emph{unlabeled} sample (introduced below) and is valued in $[0, 1]$.
In other words, we require a new \emph{unseen} unlabeled sample for the post-processing. As it will be seen from our bound, the assumption that $\heta(X, S)$ is valued in $[0, 1]$ is not restrictive, since we can always perform clipping without damaging statistical properties. 
On a more technical note, we require that $\Prob(\heta(X, S) = c \mid \heta) = 0$ almost surely for any $c \in [0, 1]$. Again, this assumption is not restrictive, since we can always randomize the output of $\heta(X, S)$
by adding a negligible noise coming from a continuous distribution. In Algorithm~\ref{alg:1} we use uniformly distributed noise supported on $[0, \sigma]$, with $\sigma$ being a small parameter. One can take this parameter $\sigma$ arbitrarily small, preserving the statistical properties of $\heta$.
\end{remark}


As mentioned before, to build the post-processing scheme, we will use only \emph{unlabeled} sample. We also do not restrict ourselves to sampling from $\Prob_{(\bsX, S)}$. Instead, we assume that for all $s \in [K]$ we observe $\{X_i\}_{i \in \class{I}_s}$ sampled \iid from $\Prob_{\bsX | S = s}$. In the above notation, $\class{I}_s$ have cardinality $n_s$ and they form a partition of $[n]$. That is, we have that $n_1 + \ldots + n_K = n$.
The described sampling scheme is potentially appealing in situations when it is possible to gather a lot of data about the minority group without the need of labeling them.
In particular, this sampling scheme allows to set $n_1 = \ldots = n_K$, which, since we do not require labeling, is more realistic.
The conditional expectation $\Exp_{\bsX \mid S=s}$ is estimated based on the following empirical measure
    \begin{align*}
        \hat{\Prob}_{\bsX \mid S=s} = \frac{1}{n_s}\sum_{i \in \class{I}_s}\delta_{\bsX_i}\enspace.
    \end{align*}




   


\begin{algorithm}[t!]
   \caption{\texttt{Post-processing}}
   \label{alg:1}
\begin{algorithmic}[1]
   \STATE {\bfseries Input:} base estimator $\heta$, unlabeled data $\{\bsX_i\}_{i \in \class{I}_s}$ for $s \in [K]$, noise magnitude $\sigma$
   
   
   \STATE \texttt{Randomize}: 
   \FOR{$i \in \class{I}_s, s\in[K]$}
    \STATE Sample independently $\zeta_i \sim \mathcal{U}([0, \sigma])$
    \STATE Set $\heta(\bsX_i, s) \leftarrow \heta(\bsX_i, s) + \zeta_i$
   \ENDFOR
   
   \STATE \texttt{Solve}: Eq.~\eqref{eq:emp_min} based on LP formulation to get $(\hbslambda, \hbsgamma)$
   
   \STATE {\bfseries Output:} $(\hbslambda, \hbsgamma)$
\end{algorithmic}
\end{algorithm}

Before providing the proposed post-processing method, we define the empirical counterpart to the function $G$ as
\begin{align*}
    \hat{G}(\bsx, s, \bslambda, \bsgamma) =
    &\left\lvert  \frac{{p}_s}{2\baralpha}(1-2\hat\eta(\bsx,s)-\scalar{\bsgamma}{\bsone}) + \frac{\scalar{\bsgamma}{\bse_s}}{2\alpha_s} \right\rvert
    - \frac{{p}_s}{2\baralpha}(1-\scalar{\bsgamma}{\bsone}) - \scalar{\bslambda}{\bse_s} - \frac{\scalar{\bsgamma}{\bse_s}}{2\alpha_s}\enspace
\end{align*}

The post-processing classifier with abstention is given by
\begin{align}
\label{eq:classif_with_abstention}
    \hat{g}(\bsx,s) {=} \begin{cases}
       r &\text{if } \hat{G}(\bsx, s, \hbslambda, \hbsgamma) \leq 0 \\
       \mathds{1}\left(\hat\eta(\bsx,s)  {>}  \frac{1}{2} + c_{\hbsgamma, s}\right) &\text{otherwise}
    \end{cases},
\end{align}
where $c_{\hbsgamma, s} \eqdef \frac{1}{2}\big(\frac{\baralpha\hat{\gamma}_s}{\alpha_s {p}_s} - \scalar{\bsone}{\hbsgamma}\big)$ and $(\hbslambda, \hbsgamma)$ is a solution of
\begin{align}
    \label{eq:emp_min}
     \min_{(\bslambda, \bsgamma)} \ens{\scalar{\bslambda}{\bsalpha} + \sum_{s=1}^K \hat{\Exp}_{\bsX \mid S = s} 
     (\hat{G}(\bsX, s,
     \bslambda,
     \bsgamma ))_+}\enspace.
\end{align}
We summarize the proposed procedure in Algorithm~\ref{alg:1} incorporating the randomization step.
Note that there is a clear analogy between the result of Theorem~\ref{thm:optimal} and the constructed algorithm. Indeed, the latter is an empirical version of the former built via the plug-in approach. \begin{figure*}[t!]
  \centering
  \begin{subfigure}[t!]{.37\linewidth}
    \includegraphics[width=\linewidth]{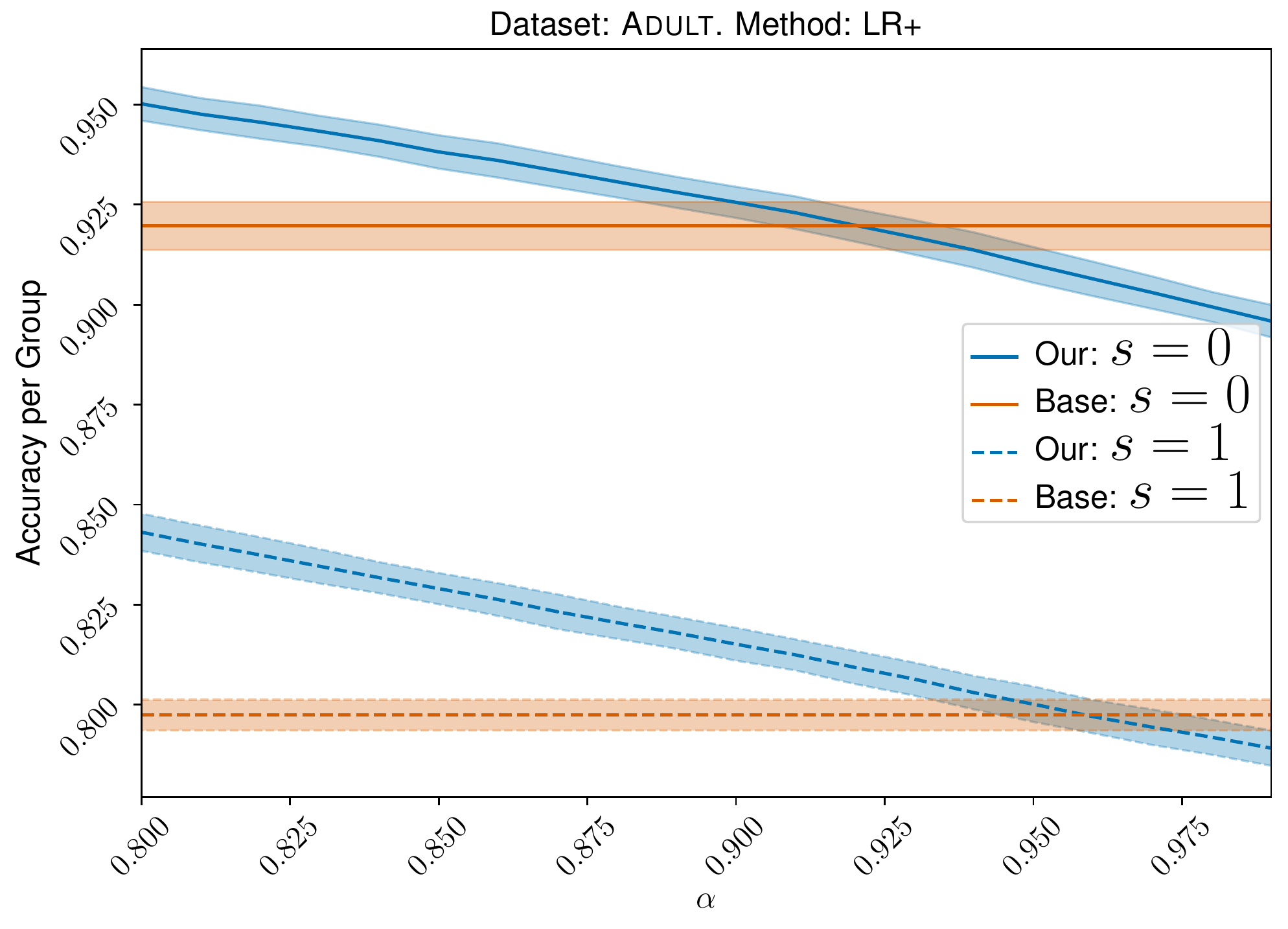}%
    \caption{Accuracy per group}\label{fig:adult2}
  \end{subfigure}%
  \begin{subfigure}[t!]{.37\linewidth}
    \includegraphics[width=\linewidth]{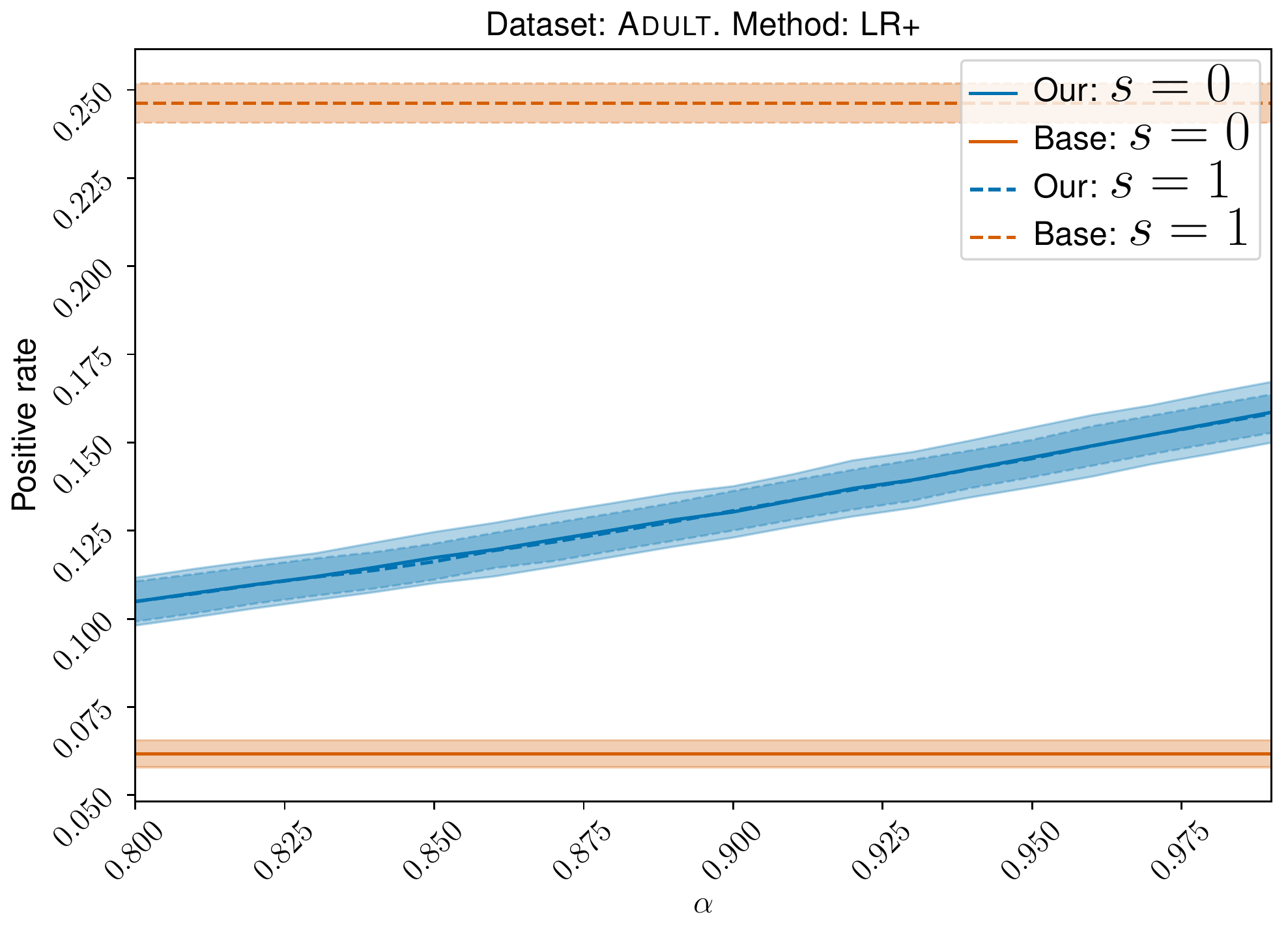}%
    \caption{Positive rate}\label{fig:adult3}
  \end{subfigure}%
  \begin{subfigure}[t!]{.26\linewidth}
    \includegraphics[width=\linewidth]{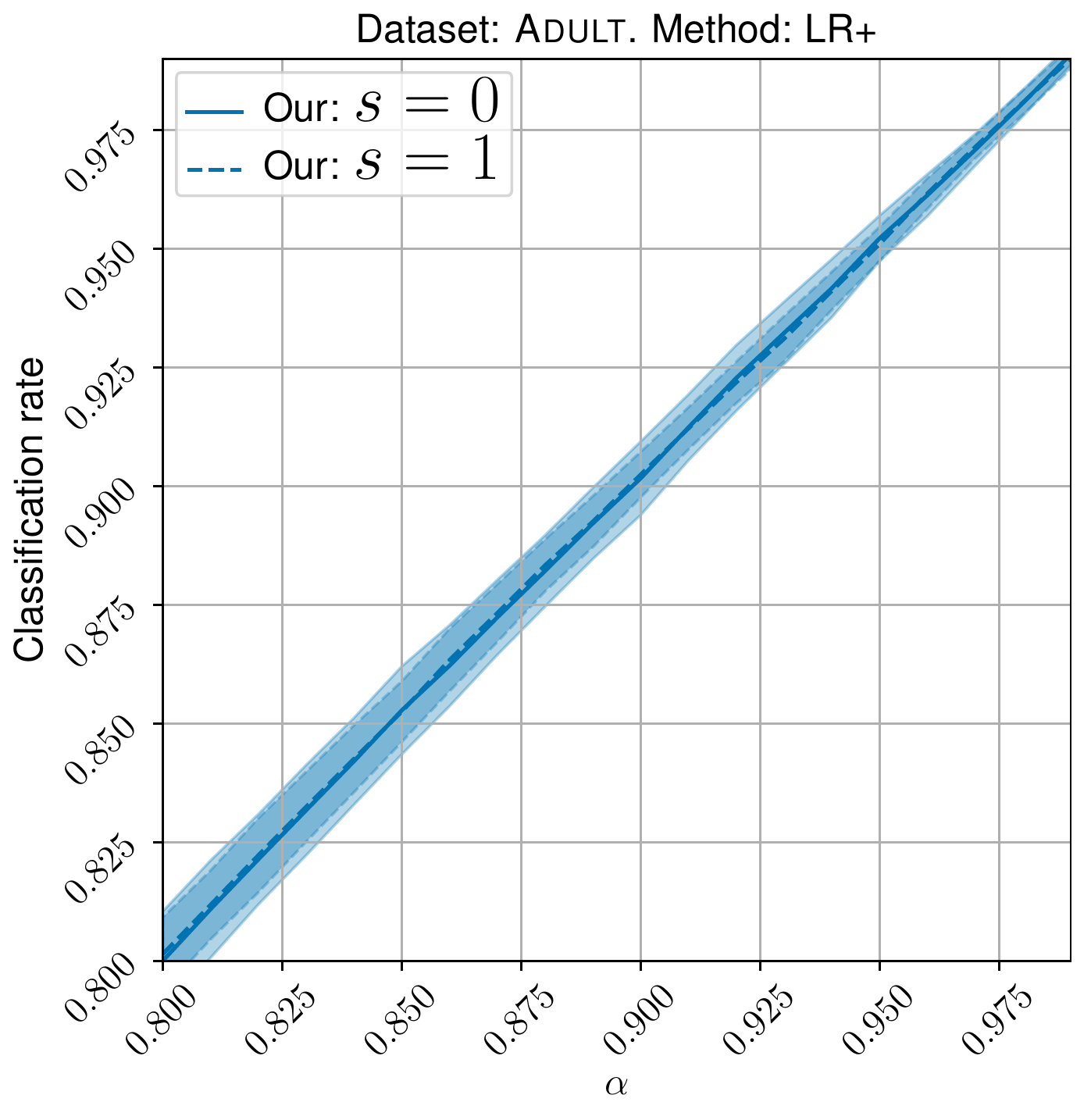}%
    \caption{Classification rate}\label{fig:adult4}
  \end{subfigure}%
  \caption{Results on \textsc{Adult} dataset with Logistic Regression (LR) as the base estimator. Blue lines correspond to our post-processing method; Orange lines correspond to the base classifier. Dashed line correspond to $s = 1$ and solid line to $s = 0$. Shaded areas correspond to the variance of the result over $20$ repetitions.}\label{fig:adult}
\end{figure*}

\begin{lemma}
The minimization problem in Eq.~\eqref{eq:emp_min} is convex and it admits a global minimizer.
\end{lemma}
In Section~\ref{sec:lp} we will actually prove a stronger statement. Namely, it will be shown that the minimization problem in Eq.~\eqref{eq:emp_min} is equivalent to a linear program with sparse constraints, which will allow us to provide an efficient implementation of the proposed procedure.

\section{Finite sample guarantees}

In this section we provide finite sample guarantees on the behavior of the post-processing classifier with abstention regarding its performance, its reject rate and its fairness. In order to lighten the presentation of our results, let us define now the sequence
\begin{align*}
    u^{\delta, K}_{n} \eqdef \sqrt\frac{2\log(\sfrac{4K}{\delta})}{2n} + \frac{2}{n}\enspace, \quad \forall n \geq 1\enspace.
\end{align*}

The sequence $u^{\delta, K}_{n}$ behaves as $O(\sqrt{\log(K / \delta) / n})$, that is, it depends logarithmically on the number of sensitive attributes $K$, on the confidence parameter $\delta$ and goes to zero as $n^{-1/2}$ with the growth of $n$.
Our goal in this section is to derive constraint and risk guarantees.
Namely, we would like to show that when $n_s \rightarrow \infty$ we have for all $s \in [K]$ that
\begin{equation*}
    \begin{aligned}
    &\left\lvert \NAB_s(\hat{g}) - \alpha_s \right\rvert \rightarrow 0\\
    &\abs{\PT_s(\hat{g}) - \PT(\hat{g})} \rightarrow 0
    \end{aligned}\qquad\text{and}\qquad
    \excess(\hat g) := \risk(\hat g) - \risk(g^*) \rightarrow 0\enspace.
\end{equation*}
The first part ensures satisfaction of reject and fairness constraints, while the second part shows that the risk of the proposed method is similar to that of $g^*$. Importantly, both guarantees will be derived in the finite-sample regime and with high probability.

The next proposition provides a quantitative control on the violation of the reject and Demographic Parity constraints in the finite sample regime. 

\begin{proposition}
\label{prop:control_reject_and_dp}
Let $\delta \in (0, 1)$. The violation of the constraints by the post-processing classifier with abstention $\hat{g}$ defined in Eq.~\eqref{eq:classif_with_abstention} can be controlled, with probability at least $1-\delta$, for any $s \in [K]$, as
\begin{align*}
    &\left\lvert \NAB_s(\hat{g}) - \alpha_s \right\rvert \leq u_{n_s}^{\sfrac{\delta}{2}, K},\quad\text{and}\quad
    \abs{\PT_s(\hat{g}) - \PT(\hat{g})} \leq \frac{6}{\alpha_s}u_{n_s}^{\delta, K} + \frac{6}{\baralpha}\sum_{s = 1}^K p_s u_{n_s}^{\delta, K}\enspace.
\end{align*}
\end{proposition}
The proof for the control of the reject rate is postponed to Section~\ref{app:control_reject} while the proof for the control of the Demographic Parity constraint can be found in Section~\ref{app:control_dp}.

Remarkably Proposition~\ref{prop:control_reject_and_dp} is assumption-free. In particular it does not depend on the conditional expectation $\eta$ as well as it does not depend on the initial estimator $\heta$. If one has enough \emph{unlabeled} data than one can get arbitrarily close to exact satisfaction of the constraints. Intuitively, this is the case because the fairness and reject constraints only depend on the conditional distribution of the feature vector $\bsX$ given the sensitive attribute $S$, not on the relation between the features and the label $Y$.

We also remark that both bounds of Proposition~\ref{prop:control_reject_and_dp} depend on the amount of observation available for each group $s \in [K]$ -- it is easier to satisfy constraints for well-represented groups.
In particular, it is advisable to collect an unlabeled sample which is balanced in terms of the sensitive attributes. Note that it is explicitly allowed in our framework, since we require samples from $\Prob_{X \mid S =s}$ and not from $\Prob_{(X, S)}$.

The next result establishes excess risk guarantees for the proposed method.
\begin{proposition}
\label{prop:main_control_risk}
Assume that $2u_{n_s}^{\delta, K}< \alpha_s <1-\sfrac{2}{n_s}$ for any $s \in [K]$ and that Assumption~\ref{as:continuous_eta} holds.
Then, for any $\delta \in (0, 1)$, the excess risk of the post-processing classifier with abstention $\hat{g}$ defined in Eq.~\eqref{eq:classif_with_abstention} satisfies with probability at least $1-\delta$, 
\begin{equation}
\label{eq:main_control_risk}
\begin{aligned}
    & \excess(\hat g) \leq
    \frac{3}{\baralpha}
    \|\eta - \heta\|_1 + 6 \sum_{s = 1}^K  \parent{\frac{p_s}{\bar\alpha} + \frac{1}{\alpha_s}} u_{n_s}^{\delta, K} \enspace.
\end{aligned}
\end{equation}
\end{proposition}

For convenience and clarity of exposition we stated separately the control on the constraint and on the excess risk. However, we remark that both Proposition~\ref{prop:control_reject_and_dp} and Proposition~\ref{prop:main_control_risk} hold on the same high-probability event.

We naturally conclude from Proposition~\ref{prop:main_control_risk} that if one has access to a consistent estimator $\hat{\eta}$ of $\eta$, \ie such that $\norm{\eta - \hat\eta}_1$ goes to $0$ as the sample sizes $(n_s)_{s = 1}^K$ go to infinity, then the excess risk can be made arbitrarily small by getting more labeled and unlabeled data.

The only assumption, constraining the reject rates $(\alpha_s)_{s=1}^K$, is quite benign. Recall that $\alpha_s$ is the rate at which the classifier is asked to give a prediction thus, in practice, it is expected to be at least greater than a half. Furthermore, note that it only depends on the size of the unlabeled dataset thus, if one has enough samples, this assumption essentially holds for free. If the sample size is small, than one has to allow the classifier to reject more often in order to satisfy the constraints.
Similar constraints are present in other contributions~\citep[see \eg][]{agarwal2018reductions,agarwal2019fair}.

Our theoretical analysis is inspired by that of \cite{chzhen2020fairplugin}. However, their results hold only in expectation while ours hold with high-probability. Moreover, due to the interplay of the reject and demographic parity constraints, their proof technique requires a non-trivial adaptation to our context. 

\begin{figure*}[t!]
  \centering
  \begin{subfigure}[t!]{.37\linewidth}
    \includegraphics[width=\linewidth]{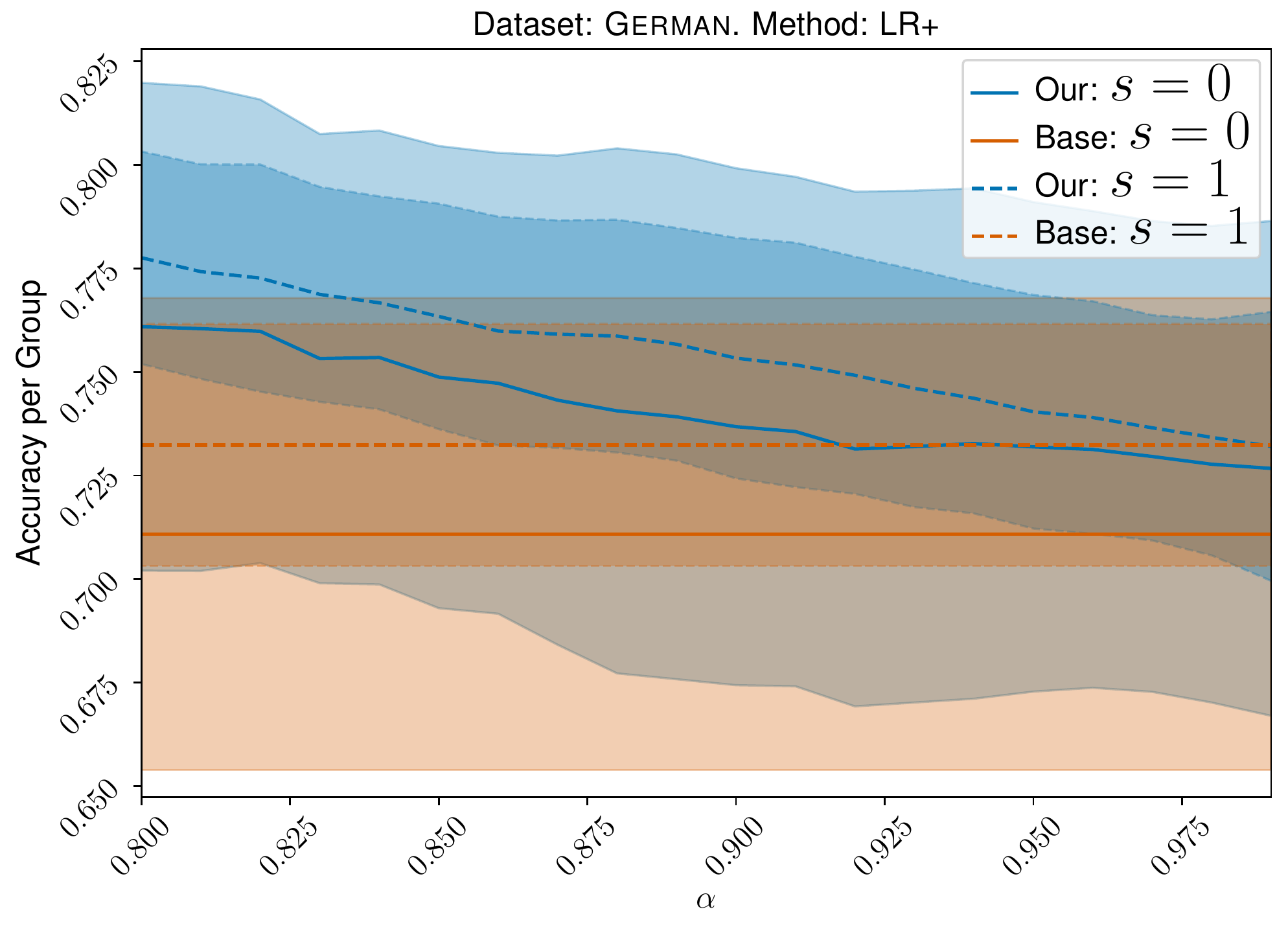}%
    \caption{Accuracy per group}\label{fig:ger2}
  \end{subfigure}%
  \begin{subfigure}[t!]{.37\linewidth}
    \includegraphics[width=\linewidth]{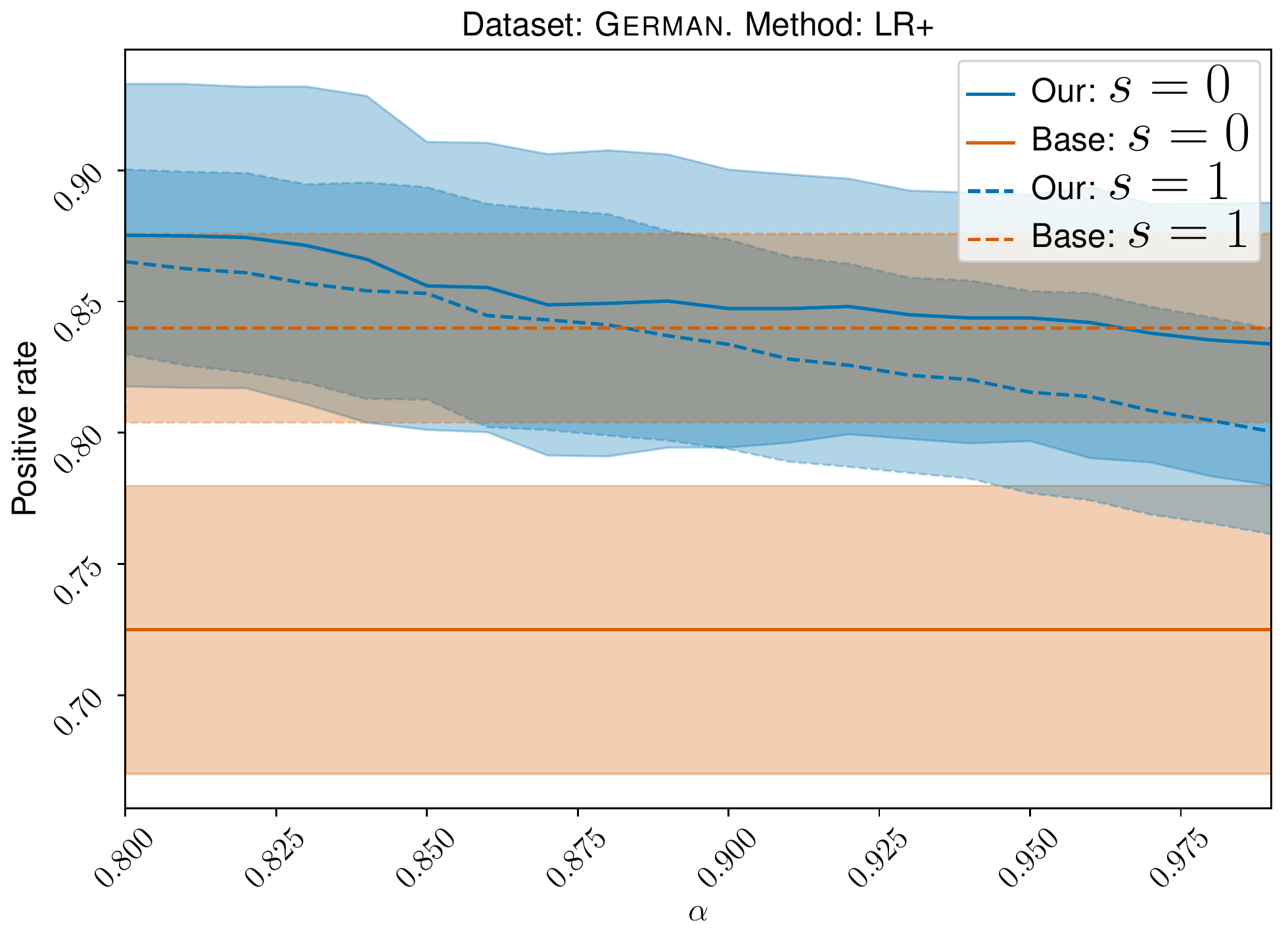}%
    \caption{Positive rate}\label{fig:ger3}
  \end{subfigure}%
  \begin{subfigure}[t!]{.26\linewidth}
    \includegraphics[width=\linewidth]{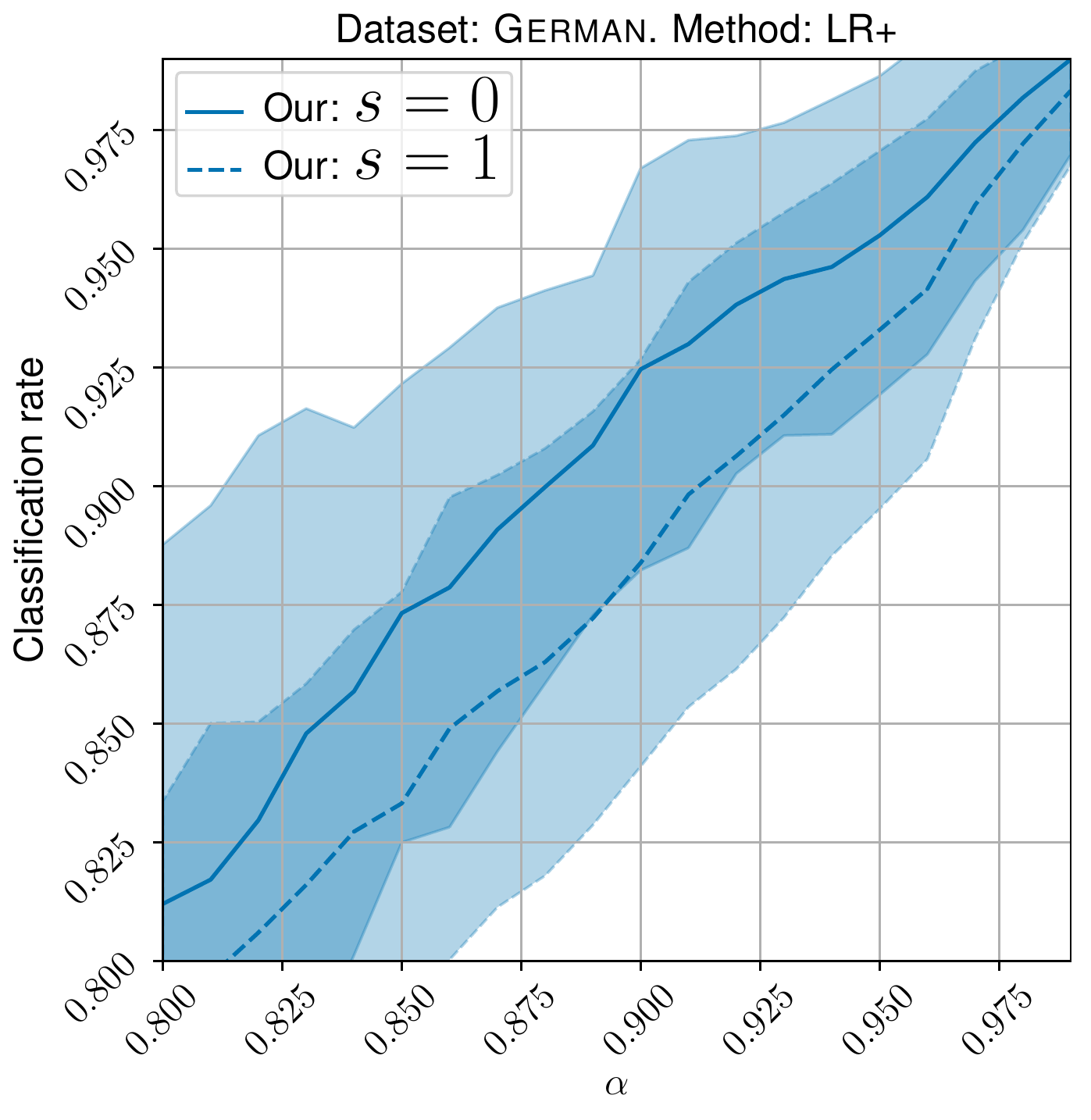}%
    \caption{Classification rate}\label{fig:ger4}
  \end{subfigure}
  \caption{Results on \textsc{German} dataset with Logistic Regression (LR) as the base estimator. Blue lines correspond to our post-processing method; Orange lines correspond to the base classifier. Dashed line correspond to $s = 1$ and solid line to $s = 0$. Shaded areas correspond to the variance of the result over $20$ repetitions.}\label{fig:german}
\end{figure*}

\section{LP reduction}
\label{sec:lp}
We recall that the proposed post-processing scheme involves solving convex non-smooth minimization problem in Eq.~\eqref{eq:emp_min}. While for low values of $K$ (few sensitive attributes) this problem can be solved via simple grid-search, which would be faster than sub-gradient methods, large values of $K$ can pose significant computational difficulties.

It turns out that the minimization problem in Eq.~\eqref{eq:emp_min} is equivalent to Linear Programming (LP)~\citep{matousek2007understanding} with sparse constraint matrix.
For any matrix $\mathbf{A} \in \bbR^{n \times m}$ we denote by $\nnZ(\mathbf{A})$ the number of non-zero elements of $\mathbf{A}$.
\begin{proposition}
\label{prop:optimization}
    There exist $\bsc \in \bbR^{n + 2K}$, $\bsb \in \bbR^{2n}$, $\mathbf{A} \in \bbR^{2n \times (n + 2K)}$ with $\nnZ(\mathbf{A}) \leq 4n + nK$, such that
    the minimization problem in Eq.~\eqref{eq:emp_min}, is equivalent to \begin{equation}
    \tag{\textbf{LP}}
    \label{eq:LP_main}
    \begin{aligned}
        &\min_{\bsy \in \bbR^{n + 2K}}\scalar{\bsc}{\bsy}\\
        &\text{s.t.}\qquad
        \begin{cases}
            \mathbf{A}\bsy \leq \bsb &\\
            y_i \geq 0 &i \in [n]
        \end{cases}\enspace.
    \end{aligned}
    \end{equation}
\end{proposition}
Due to the space considerations, the previous result is stated in existential form, however, all the parameters of the LP are explicit and are provided in the supplementary material.
Seminal works of~\citep{khachiyan1979polynomial,karmarkar1984new} confirmed that LP with rational coefficients can be solved in weakly polynomial time. Since then, extremely efficient solvers were developed based on the interior-point and simplex methods. 
The fact that the post-processing reduces to an LP problem allows us to use these fast solvers. In particular, most of the computational burden lies on the training of the base estimator $\heta$ while the post-processing can be performed almost instantly.
From theoretical perspective, one can leverage the sparse structure of the problem using, for instance, the result of~\citep{lee2015efficient} who provide an efficient solver to find an $\varepsilon$ solution of an LP in $\tilde{O}((\nnZ(\mathbf{A}) + n^2)\sqrt{n}\log(\varepsilon^{-1}))$ time. In particular, the previous guarantee scales only linearly with the number of sensitive attributes and logarithmically with the precision $\varepsilon$. However, in our practical implementation of the proposed method, we use interior point method available as a part of \texttt{scipy.optimize.linprog}~\citep{scipy}. 

\section{Experiments}
\label{sec:exps}

We provide an implementation of the proposed post-processing procedure described in Algorithm~\ref{alg:1} using \texttt{scipy.optimize.linprog}~\citep{scipy}, which implements interior point method for solving problem~\eqref{eq:LP_main}. The source code is available at~\url{https://github.com/evgchz/dpabst}.
We consider \textsc{Adult}~\citep{adult} and
\textsc{German}~\citep{UCI} datasets, which are standard benchmark datasets in the fairness literature.

\textsc{Adult} dataset is fetched via \texttt{fairlearn.datasets} \citep{fairlearn}. This dataset contains $14$ features and around $48,000$ observations. We dropped those observations that contain missing values. This dataset consists of 1994 US Census entries. Each entry of this dataset corresponds to an individual who is described by $14$ characteristics, the binary target variable is equal to $1$ if the individual earns more than $\$50K$ per year and it is set to $0$ otherwise. In our experiments we take sex as a sensitive attribute.

\textsc{German} dataset is hosted on the UCI Machine Learning Repository \citep{UCI}. Each of the $1,000$ entries represents a person who takes a credit by a bank. The binary target variable is equal to one if the individual is considered as good credit risks based on $20$ categorical/symbolic attributes and is set to $0$ otherwise.
We use ordinal-encoding for ordinal variables and one-hot-encoding for other categorical variables which yields $46$ features in total. In our experiments we take sex as sensitive attribute.

We consider the following off-the-shelf methods: Random Forest (RF) and Logistic Regression (LR).
We used the \texttt{sklearn}~\citep{sklearn} implementation of the aforementioned methods.

Each dataset of size $N$ we partition in three parts. The first labeled part ($60\%$ of $N$) is used to train the base classifier, the second unlabeled part ($20\%$ of $N$) is used to apply the proposed post-processing, and the third part ($20\%$ of $N$) is used for evaluation of various statistics, which describe performance of the algorithm.

The hyperparameters of each base algorithm are tuned via $5$-fold cross validation with accuracy as the performance measure. The regularization parameter of LR is searched among $30$ values, equally spaced in logarithmic scale between $10^{-4}$ and $10^{4}$. For RF the number of trees has been set to $1000$ and the size of the subset of features optimized at each node has been searched in $\{d, \ceil*{d^{\sfrac{15}{16}}}, \ceil*{d^{\sfrac{7}{8}}}, \ceil*{d^{\sfrac{3}{4}}}, \ceil*{d^{\sfrac{1}{2}}}, \ceil*{d^{\sfrac{1}{4}}}, \ceil*{d^{\sfrac{1}{8}}}, \ceil*{d^{\sfrac{1}{16}}}, 1\}$ where $d$ is the number of features in the dataset.
Recall that our post-processing algorithm is parameter-free, thus, the second step is performed without any tuning.
Our setup allows to set different reject rates for different groups. However, the exact values heavily depend on the domain specific knowledge and on the problem itself. Because of that, in our experiments, we set $\alpha_1 = \ldots = \alpha_K = \alpha$ for 20 values of $\alpha$ taking values in the uniform grid over $[.8, .99]$, which correspond to reject rate ranging from $20\%$ to $1\%$.

Given a classifier with reject option $g$ and a test data $\class{T} = \{(\bsx_i, s_i, y_i)\}_{i = 1}^{n_{\test}}$, we evaluate the following statistics
\begin{align*}
    &\widehat\acc_s(g) = \frac{\sum_{i = 1}^{n_{\test}}\ind{g(\bsx_i, s_i) = y_i}\ind{s_i = s}}{\sum_{i = 1}^{n_{\test}}\ind{g(\bsx_i, s_i) \neq r}\ind{s_i = s}}\enspace,&&s = 1, \ldots, K\enspace,\\
    &\widehat\clf_s(g) = \frac{\sum_{i = 1}^{n_{\test}}\ind{g(\bsx_i, s_i) \neq r}\ind{s_i = s}}{\sum_{i = 1}^{n_{\test}}\ind{s_i = s}}\enspace, &&s = 1, \ldots, K\enspace,\\
    &\widehat\pos_s(g) = \frac{\sum_{i = 1}^{n_{\test}}\ind{g(\bsx_i, s_i) = 1}\ind{s_i = s}}{\sum_{i = 1}^{n_{\test}}\ind{g(\bsx_i, s_i) \neq r}\ind{s_i = s}}\enspace, &&s = 1, \ldots, K\enspace.
\end{align*}
The first statistic measures the accuracy of $g$, the second the group-wise classification rate of $g$, and the third one measures the group-wise predicted positive rate of $g$. It is important to keep in mind that a classifier $g$ which never rejects achieves $\clf_s(g) = 1$ on any dataset.

\begin{figure*}[t!]
  \centering
  \begin{subfigure}[t!]{.37\linewidth}
    \includegraphics[width=\linewidth]{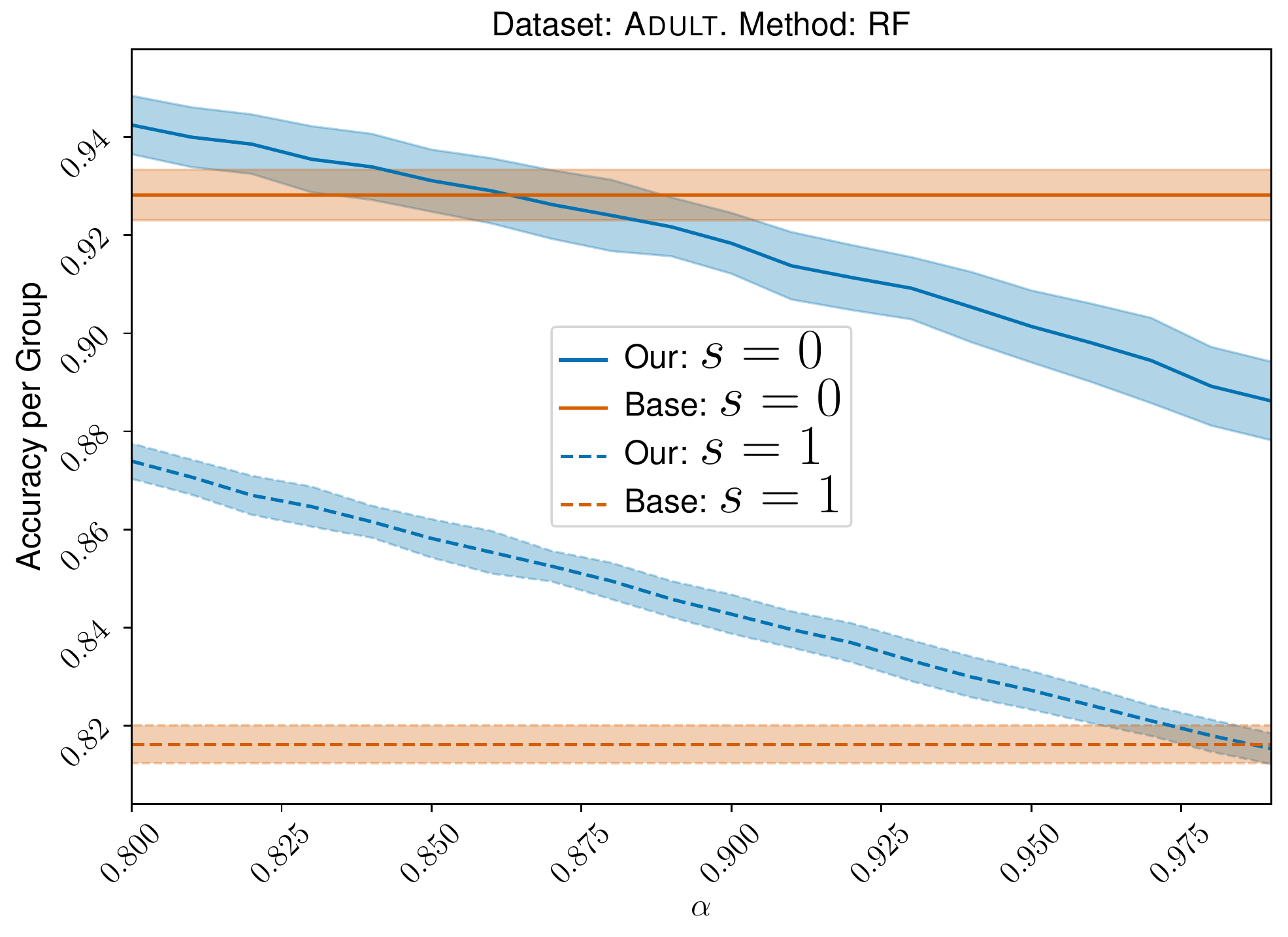}%
    \caption{Accuracy per group}\label{fig:adultrf2}
  \end{subfigure}%
  \begin{subfigure}[t!]{.37\linewidth}
    \includegraphics[width=\linewidth]{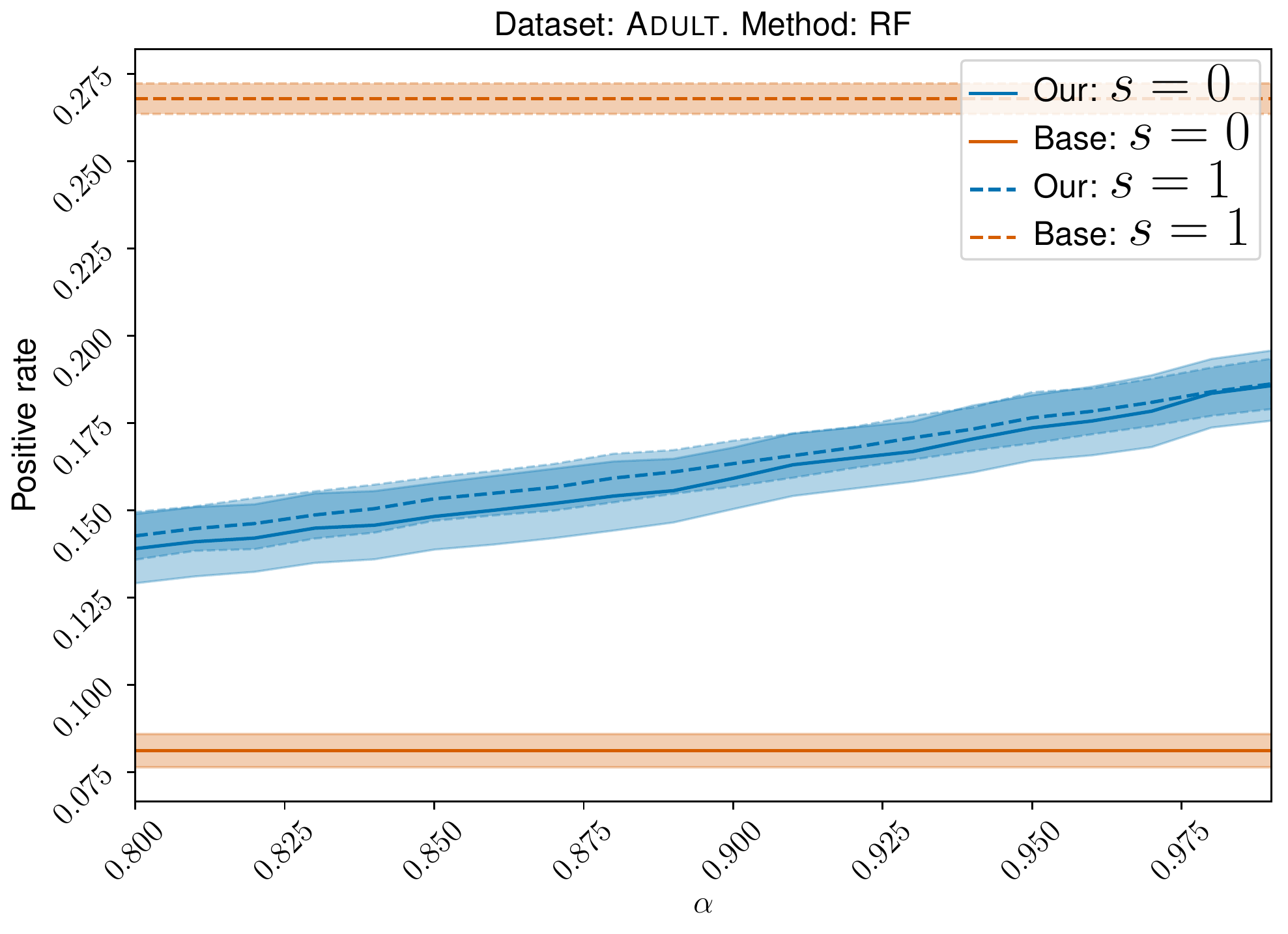}%
    \caption{Positive rate}\label{fig:adultrf3}
  \end{subfigure}%
  \begin{subfigure}[t!]{.26\linewidth}
    \includegraphics[width=\linewidth]{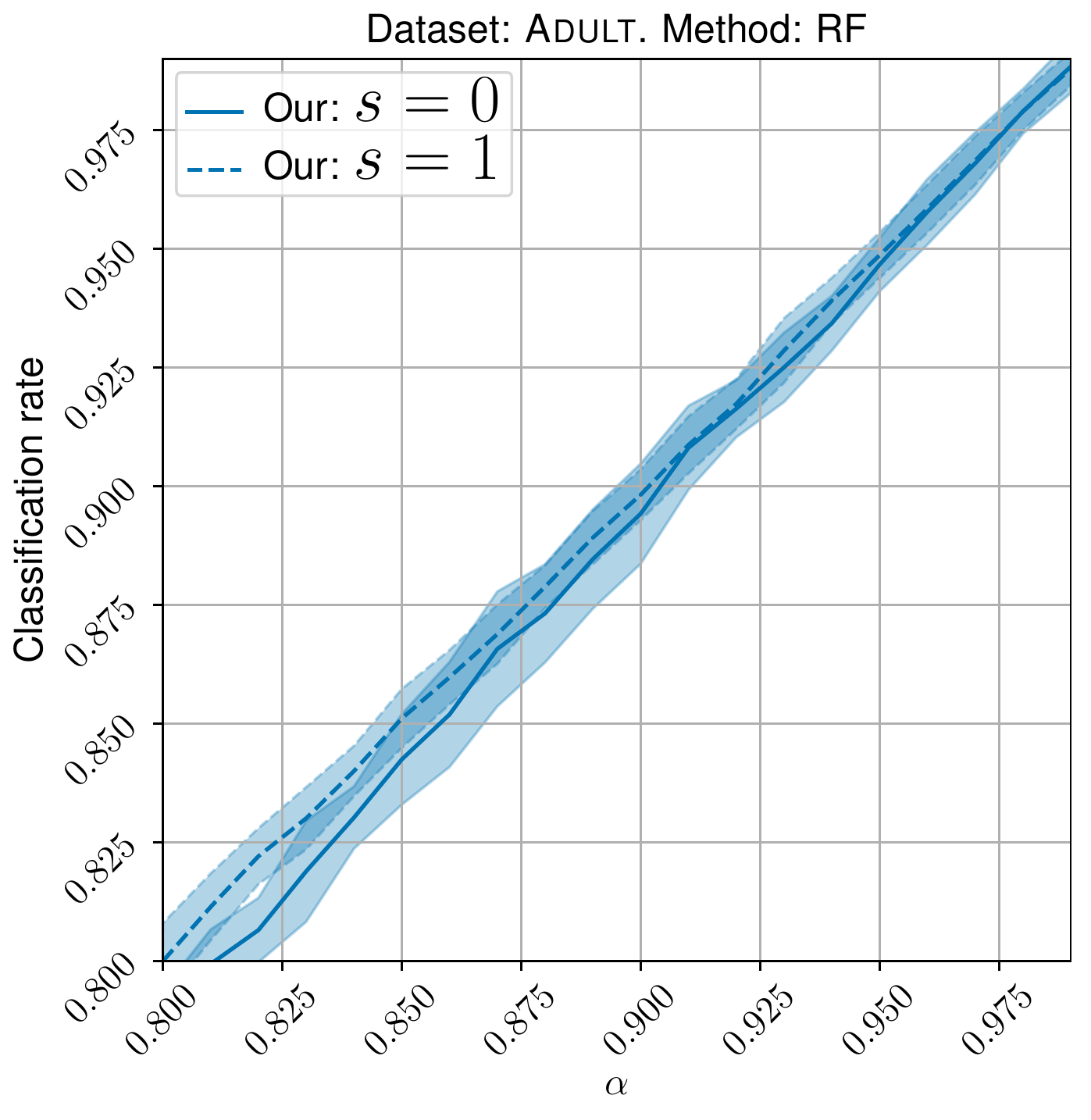}%
    \caption{Classification rate}\label{fig:adultrf4}
  \end{subfigure}
  \caption{Results on \textsc{Adult} dataset with Random Forest (RF) \textbf{without} additional randomization as the base estimator. Blue lines correspond to our post-processing method; Orange lines correspond to the base classifier. Dashed line correspond to $s = 1$ and solid line to $s = 0$. Shaded areas correspond to the variance of the result over $20$ repetitions.}\label{fig:adultrf}
\end{figure*}

\begin{figure*}[t!]
  \centering
  \begin{subfigure}[t!]{.37\linewidth}
    \includegraphics[width=\linewidth]{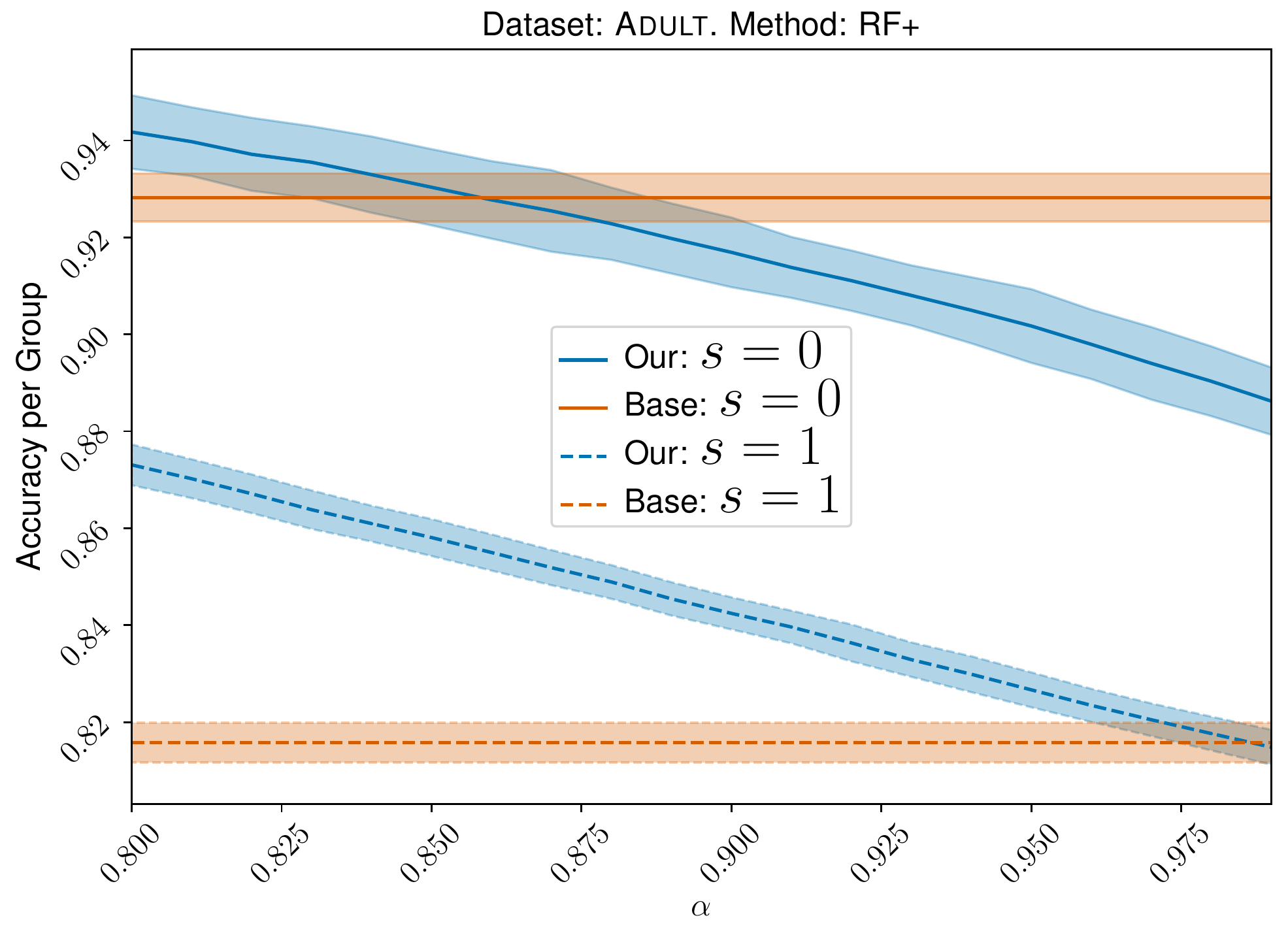}%
    \caption{Accuracy per group}\label{fig:adultrf+2}
  \end{subfigure}%
  \begin{subfigure}[t!]{.37\linewidth}
    \includegraphics[width=\linewidth]{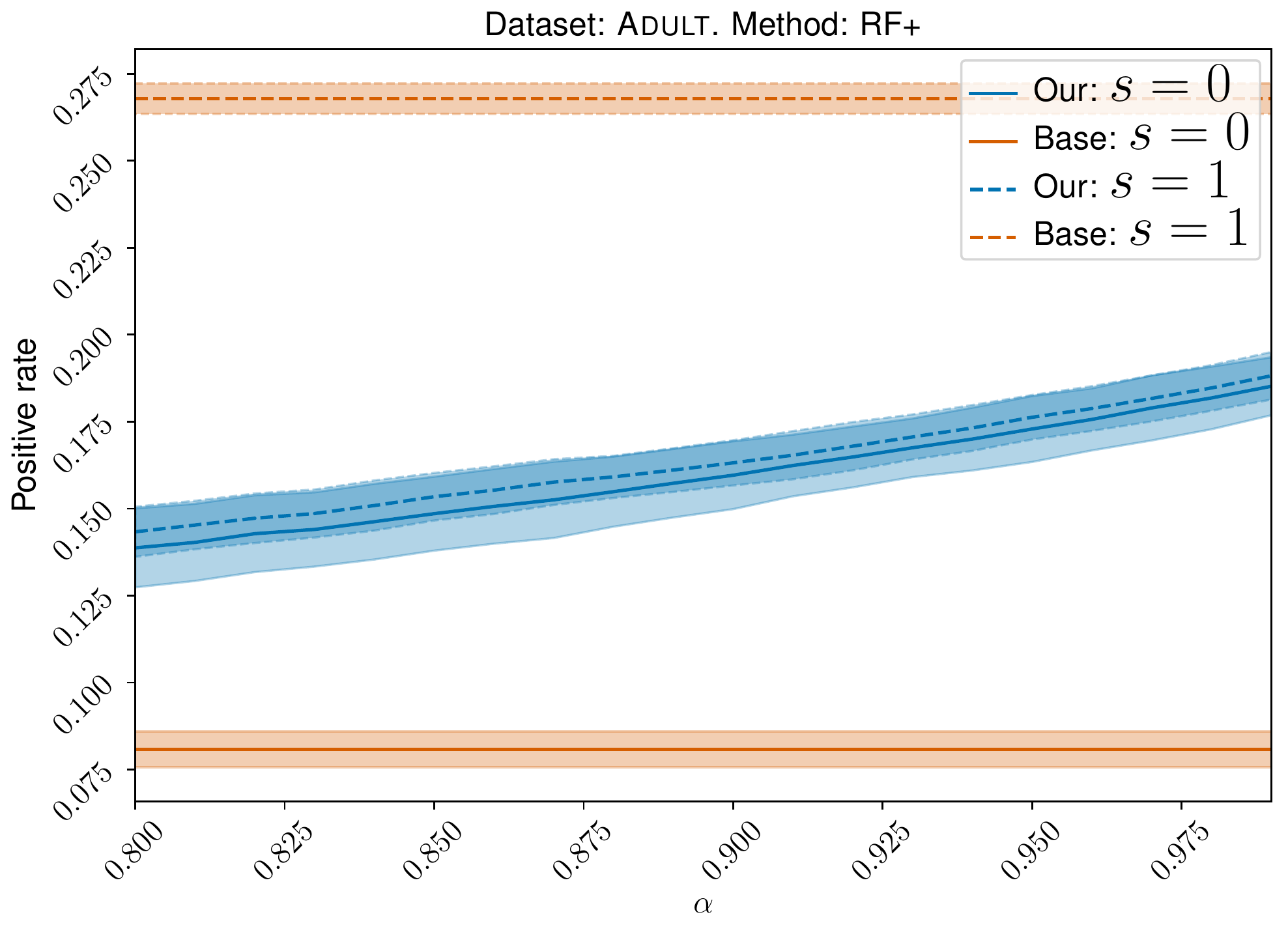}%
    \caption{Positive rate}\label{fig:adultrf+3}
  \end{subfigure}%
  \begin{subfigure}[t!]{.26\linewidth}
    \includegraphics[width=\linewidth]{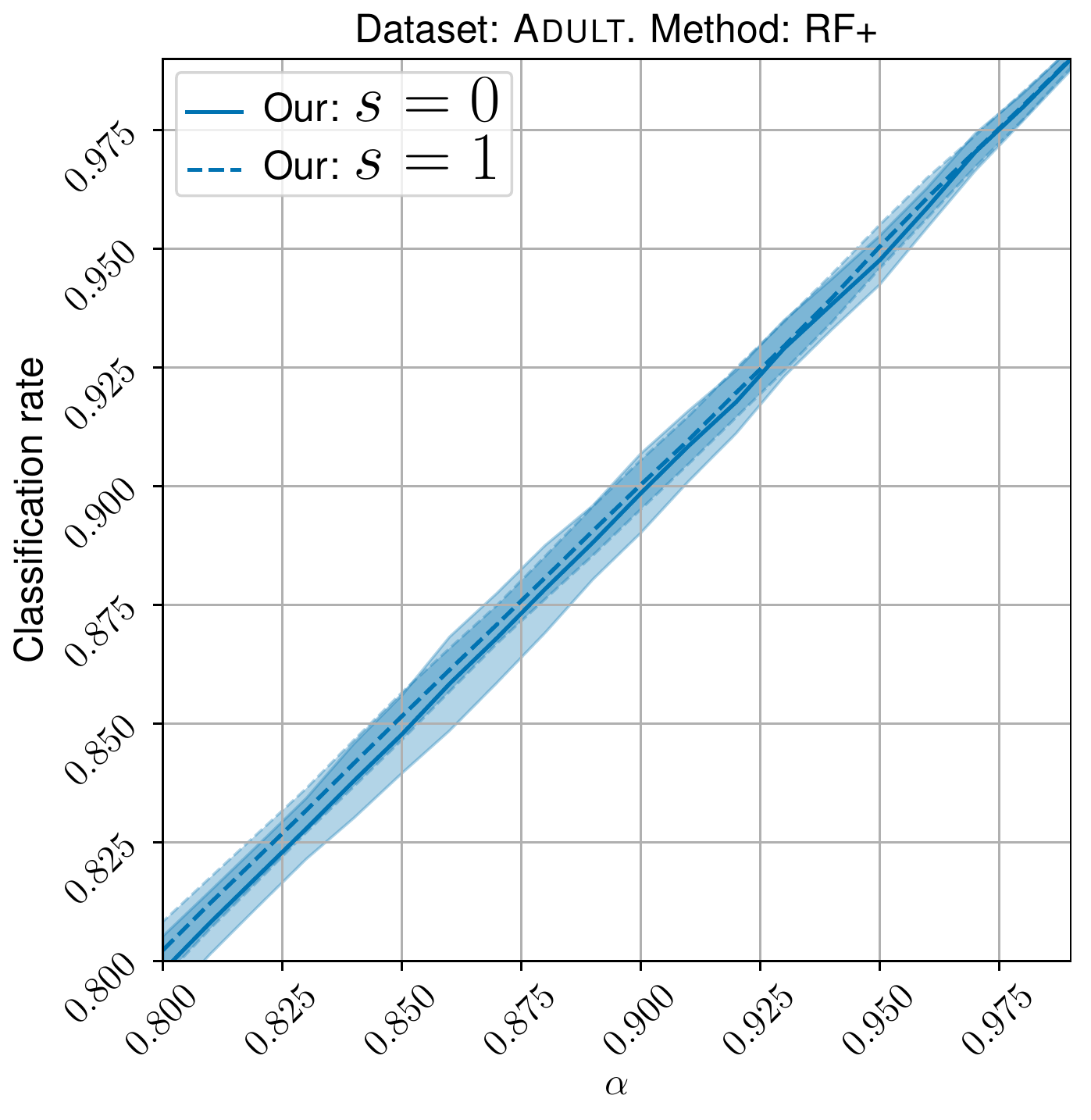}%
    \caption{Classification rate}\label{fig:adultrf+4}
  \end{subfigure}
  \caption{Results on \textsc{Adult} dataset with Random Forest (RF) \textbf{with} additional randomization as the base estimator. Blue lines correspond to our post-processing method; Orange lines correspond to the base classifier. Dashed line correspond to $s = 1$ and solid line to $s = 0$. Shaded areas correspond to the variance of the result over $20$ repetitions.}\label{fig:adultrf+}
\end{figure*}
Figure~\ref{fig:adult} presents results on \textsc{Adult} dataset. First of all we observe that the proposed post-processing is effective in imposing reject and fairness constraints as illustrated on Figures~\ref{fig:adult3}-\ref{fig:adult4}. Looking at Figure~\ref{fig:adult2}, we observe that for already moderately low values of rejection our classification algorithm equalizes and even exceeds the accuracy per groups and overall of the base classifier.
Figure~\ref{fig:german} presents result on \textsc{German} dataset. Overall conclusions remain the same as for the \textsc{Adult} dataset. The main difference is an increase in variance of the result. This effect should not be attributed to the method itself but rather to the size of the two datasets. Indeed, \textsc{Adult} contains around $40,000$ observation, while \textsc{German} contains only $1,000$ observations. Hence, it is simply a more difficult task to learn stable classification algorithms on the \textsc{German} dataset. Remarkably, already $1\%$ of reject rate allows to maintain the accuracy of the base classifier while significantly improving its fairness as illustrated on Figure~\ref{fig:ger2}.

We would also like to highlight the importance of the additive noise perturbation present in Algorithm~\ref{alg:1}. To this end, we consider RF classifier, which naturally does not lead to continuous estimator $\heta(\bsX, S)$ due to its partitioning nature.
On Figure~\ref{fig:adultrf} we display the performance of our algorithm without any additional randomization and on Figure~\ref{fig:adultrf+} follow Algorithm~\ref{alg:1} with $\sigma = 10^{-3}$. One can see that on Figure~\ref{fig:adultrf4} the behaviour of our procedure fails to satisfy rejection rate constraints for lower values of $\alpha$, even, considering the fact, that we have a rather large dataset. In contrast, this phenomenon disappears once the noise is added (see Figure~\ref{fig:adultrf+4}), confirming our theoretical findings.
It is important to emphasize that this additional randomization has only a little impact on the group-wise accuracy, which suggest that the randomization step is always advisable in practice.

\section{Conclusion}
We proposed a classification with abstention algorithm which is able to satisfy Demographic Parity and whose reject rate is controlled explicitly. Our procedure is based on a post-processing scheme of any base estimator and can be computed efficiently using LP solvers.
We derived distribution-free finite-sample guarantees demonstrating that the proposed method is able to achieve the prescribed constraints with high probability. Under additional mild assumption, we showed the risk of the proposed procedure nearly matches that of the theoretical  minimum, provided the initial estimator is consistent.
Our experimental results support the developed theory and suggest that by allowing small reject rate it is possible to avoid the accuracy-fairness trade-off.

\section{Acknowledgements}
This work was supported by a public grant as part of the Investissement d'avenir project, reference ANR-11-LABX-0056-LMH, LabEx LMH.

\bibliography{example_paper}

\appendix

\section*{Structure of Appendix}

Appendix~\ref{app:optimal_classif} is devoted to the proof of Theorem~\ref{thm:optimal}. Appendix~\ref{app:aux_results} reminds and proves auxiliary results that are used in the rest of the supplementary material.
The proof of Proposition~\ref{prop:control_reject_and_dp} is split across Appendix~\ref{app:control_reject} for the control of the reject rate and Appendix~\ref{app:control_dp} for the control of the demographic parity violation. Appendix~\ref{app:control_risk} contains the proof of  Proposition~\ref{prop:main_control_risk}. Finally, Appendix~\ref{app:optimization} provides a constructive proof of Proposition~\ref{prop:optimization}.

\section{Derivation of the optimal prediction}
\label{app:optimal_classif}

Recall that we are interested in solving the following problem
\begin{align*}
    \min_{g: \bbR^d \times [K] \to \{0, 1, r\}}&\,\Prob(g(\bsX, S) \neq Y \mid g(\bsX, S)) \neq r)\\
    &\text{s.t.}
    \begin{cases}
       \Prob(g(\bsX, S)) \neq r \mid S = s) = \alpha_s,\quad \forall s \in [K]\\
       \Prob(g(\bsX, S) = 1 \mid S = s, g(\bsX, S) \neq r) = \Prob(g(\bsX, S) = 1 \mid g(\bsX, S)) \neq r), \quad \forall s \in [K]
    \end{cases}\enspace.
\end{align*}


\subsection{Simplifications}

First we simplify the quantities involved in the above problem. Set $\baralpha = \sum_{s=1}^K p_s \alpha_s$ and recall that we defined the random variable $\eta(\bsX, S) = \mathbb{E}[Y\mid \bsX, S]$. Observe that for any $g$ such that $\Prob(g(\bsX, S) \neq r \mid s = s) = \alpha_s$, we can write
\begin{align*}
    &\Prob(g(\bsX, S) \neq Y \mid g(\bsX, S) \neq r) = \sum_{s=1}^K \frac{p_s}{\baralpha} \Exp_{\bsX | S=s}\left[(1-\eta(\bsX, S)\mathds{1}_{g(\bsX, S) = 1} + \eta(\bsX, S) \mathds{1}_{g(\bsX, S) = 0}  \right]\enspace,\\
    &\Prob(g(\bsX, S) \neq r | S=s) = \Exp_{\bsX | S=s}\left[\mathds{1}_{g(\bsX, S) = 1} + \mathds{1}_{g(\bsX, S) = 0} \right]\enspace,\\
    &\Prob(g(\bsX, S)=1 \mid g(\bsX, S)\neq r) = \sum_{s=1}^K \frac{p_s}{\baralpha}  \Exp_{\bsX | S=s}\left[ \mathds{1}_{g(\bsX, S) = 1} \right]\enspace,\\
    &\Prob(g(\bsX, S)=1 | S=s, g(\bsX, S)\neq r) = \frac{1}{\alpha_s} \Exp_{\bsX | S=s}\left[ \mathds{1}_{g(\bsX, S) = 1} \right]\enspace.
\end{align*}

\subsection{Lagrangian}

We introduce the Lagrangian $\class{L}$ of the constrained minimization problem as
\begin{align*}
    \lag(g, \bslambda, \bsgamma) &= \Prob(g(\bsX, S) \neq Y \mid g(\bsX, S) \neq r) + \sum_{s=1}^K \lambda_s ( \Prob(g(\bsX, S) \neq r | S=s) - \alpha_s) \\ &+ \sum_{s=1}^K\gamma_s(\Prob(g(\bsX, S)=1 | S=s, g(\bsX, S)\neq r) -  \Prob(g(\bsX, S)=1 | g(\bsX, S)\neq r))\enspace.
\end{align*}
Using the simpler expressions we derived earlier, the Lagrangian can be expressed as
\begin{align*}
    \lag(g, \bslambda, \bsgamma) &=  \sum_{s=1}^K \frac{p_s}{\baralpha} \Exp_{\bsX | S=s}\left[(1-\eta(\bsX, S))\mathds{1}_{g(\bsX, S) = 1} + \eta(\bsX, S) \mathds{1}_{g(\bsX, S) = 0}  \right]\\
    &+ \sum_{s=1}^K \lambda_s \left\{ \Exp_{\bsX | S=s}\left[\mathds{1}_{g(\bsX, S) = 1} + \mathds{1}_{g(\bsX, S) = 0} \right] - \alpha_s  \right\}\\
    &+ \sum_{s=1}^K \frac{\gamma_s}{\alpha_s} \Exp_{\bsX | S=s}\left[\mathds{1}_{g(\bsX, S) = 1}\right] - \left(\sum_{s'=1}^K \gamma_{s'}\right) \left(\sum_{s=1}^K \frac{p_s}{\baralpha} \Exp_{\bsX | S=s}\left[\mathds{1}_{g(\bsX, S) = 1}\right]\right)\enspace.
\end{align*}
After straightforward algebraic manipulations, the Lagrangian can be simplified to
\begin{align*}
    \lag(g, \bslambda, \bsgamma) = \sum_{s=1}^K \Exp_{\bsX | S=s}\left[H_{(\bsX, s)}(g, \bslambda, \bsgamma)  \right] - \sum_{s=1}^K \lambda_s \alpha_s\enspace,
\end{align*}
where,  setting $\bar\gamma \eqdef \sum_{s=1}^K \gamma_s$, we defined the function 
\begin{align*}
    H_{(\bsx,s)}(g, \bslambda, \bsgamma) = 
    \begin{cases}
       0, &\text{ if } g(\bsx,s)=r\\
       \frac{p_s}{\baralpha} \eta(\bsx, s) + \lambda_s, &\text{ if } g(\bsx,s)=0\\
       \frac{p_s}{\baralpha}(1-\eta(\bsx,s) - \bar{\gamma}) + \lambda_s + \frac{\gamma_s}{\alpha_s}, &\text{ if } g(\bsx,s)=1
    \end{cases}\enspace.
\end{align*}

Using this Langrangian, our initial problem can be expressed as
\begin{align*}
    \min_{g}\max_{(\bslambda, \bsgamma) \in \bbR^K \times \bbR^K}\class{L}(g, \bslambda, \bsgamma)\enspace.
\end{align*}
Weak duality then implies that
\begin{align*}
   \min_{g}\max_{(\bslambda, \bsgamma) \in \bbR^K \times \bbR^K}\class{L}(g, \bslambda, \bsgamma) \geq \max_{(\bslambda, \bsgamma) \in \bbR^K \times \bbR^K}\min_{g}\class{L}(g, \bslambda, \bsgamma)\enspace.
\end{align*}
\paragraph{Dual problem.} We first solve the inner minimization  problem of the $\max\min$ formulation for any $(\bslambda, \bsgamma)$,
\begin{align}
    \label{eq:weak_dual}
    \min_{g}\class{L}(g, \bslambda, \bsgamma)\enspace,
\end{align}
and then show that strong duality holds under our assumptions.
The problem in Eq.~\eqref{eq:weak_dual} can be solved point-wise, that is, it is sufficient to solve
\begin{align*}
    \min_{z \in \{0, 1, r\}} H_{(\bsx,s)}(z, \bslambda, \bsgamma)\enspace,
\end{align*}
for any $s \in [K]$ and any $\bsx \in \mathbb{R}^d$.
One can easily check that, for any given couple $(\bsx,s)$, the minimizer of the above expression is given by
\begin{align*}
    \tilde{g}(\bsx,s) = \begin{cases}
       r, &\text{ if } 0 \leq \min(\frac{p_s}{\baralpha}\eta(\bsx,s) + \lambda_s, \frac{p_s}{\baralpha}(1-\eta(\bsx, s) - \bar{\gamma}) + \lambda_s + \frac{\gamma_s}{\alpha_s})\\
       \mathds{1}(\frac{p_s}{\baralpha}(1-2\eta(\bsx,s) - \bar{\gamma}) + \frac{\gamma_s}{\alpha_s} < 0), &\text{ otherwise}
    \end{cases}\enspace.
\end{align*}
Note that, using the fact that $2\min(a,b) = a+b - \lvert a - b \rvert$, the previous expression simplifies to
\begin{align*}
    \tilde{g}(\bsx,s) = \begin{cases}
       r, &\text{ if } \left\lvert \frac{p_s}{2\baralpha}(1-2\eta(\bsx,s) - \bar\gamma) + \frac{\gamma_s}{2\alpha_s}\right\rvert \leq \lambda_s + \frac{p_s}{2\baralpha}(1-\bar\gamma) + \frac{\gamma_s}{2\alpha_s} \\
       \mathds{1}(\frac{p_s}{\baralpha}(1-2\eta(\bsx,s) - \bar{\gamma}) + \frac{\gamma_s}{\alpha_s} < 0), &\text{ otherwise}
    \end{cases}
    \enspace.
\end{align*}
Plugging back the expression for $\tilde{g}$ in the function $H$ we get
\begin{align*}
    H_{(\bsx,s)}(\tilde{g}, \bslambda, \bsgamma) = \left( \frac{p_s}{2\baralpha}(1-\bar\gamma)  + \lambda_s + \frac{\gamma_s}{2\alpha_s} - \left\lvert  \frac{p_s}{2\baralpha}(1-2\eta(\bsx, s)-\bar\gamma) + \frac{\gamma_s}{2\alpha_s} \right\rvert \right)_-\enspace,
\end{align*}
where $(a)_-\eqdef \min(a, 0)$.
Substituting this expression into the Lagrangian, we can derive the dual optimization problem as
\begin{align*}
    \max_{(\bslambda, \bsgamma)} \left\{ \sum_{s=1}^K \Exp_{\bsX|S=s}\left(\frac{p_s}{2\bar\alpha}(1-\bar\gamma) + \lambda_s + \frac{\gamma_s}{2\alpha_s} - \left\lvert  \frac{p_s}{2\bar\alpha}(1-2\eta(\bsx, s)-\bar\gamma) + \frac{\gamma_s}{2\alpha_s} \right\rvert \right)_- - \sum_{s=1}^K \lambda_s \alpha_s\right\}\enspace.
\end{align*}

Writing this optimization problem as a minimization problem in vector form, the optimal Lagrange multipliers $(\bslambda^*, \bsgamma^*)$ are a solution of
\begin{align}
    \label{eq:min_true}
    \min_{(\bslambda, \bsgamma)}\left\{\sum_{s=1}^K  \Exp_{\bsX|S=s}\left(\left\lvert  \frac{p_s}{2\bar\alpha}(1-2\eta(\bsX,S)-\scalar{\bsgamma}{\bsone}) + \frac{\scalar{\bsgamma}{\bse_s}}{2\alpha_s} \right\rvert  - \frac{p_s}{2\bar\alpha}(1-\scalar{\bsgamma}{\bsone}) - \scalar{\bslambda}{\bse_s} - \frac{\scalar{\bsgamma}{\bse_s}}{2\alpha_s}\right)_+ +  \scalar{\bslambda}{\bsalpha}\right\}\enspace,
\end{align}
where for any real number $y$, $(y)_+ \eqdef \max(x, 0)$ and for any $s\in [K]$, $\bse_s$ is the $s$-basis vector of $\bbR^K$.

Let us check that the objective function of the above optimization problem is jointly convex in $(\bslambda, \bsgamma)$. First of all, the mappings
\begin{align*}
    (\bslambda, \bsgamma) &\mapsto \frac{p_s}{2\bar\alpha}(1-2\eta(\bsx, s)-\scalar{\bsgamma}{\bsone}) + \frac{\scalar{\bsgamma}{\bse_s}}{2\alpha_s}\enspace,\\
    (\bslambda, \bsgamma) &\mapsto - \frac{p_s}{2\bar\alpha}(1-\scalar{\bsgamma}{\bsone}) - \scalar{\bslambda}{\bse_s} - \frac{\scalar{\bsgamma}{\bse_s}}{2\alpha_s}\enspace,
\end{align*}
are clearly affine mappings. Since taking the absolute value of an affine mapping gives a convex mapping (as a maximum between two affine, hence convex, functions), the sum of the absolute value of the first mapping with the second mapping is a convex function.
Furthermore, the composition with the positive part function preserves convexity since this operation can be expressed as taking the maximum between two convex functions. 
Finally, by linearity of expectation, we notice that the objective is expressed as a finite sum of convex functions and conclude that it is jointly convex in $(\bslambda, \bsgamma)$.


The objective function is not smooth everywhere due to the presence of absolute values and positive part functions. However, thanks to Assumption~\ref{as:continuous_eta}, the set of points at which the objective function is not differentiable has zero Lebesgue measure and can thus be ignored.
The First-Order Optimality Conditions (FOOC) on the optimal Lagrange multipliers $(\bslambda^*, \bsgamma^*)$ then read as
\begin{equation}
\label{eq:KKT_lambda_star}
\tag{\textbf{FOOC}}
    \begin{aligned}
    \alpha_s &= \Prob_{\bsX\mid S=s}\left( \left\lvert  \frac{p_s}{2\bar\alpha}(1-2\eta(\bsX,s)-\scalar{\bsgamma^*}{\bsone}) + \frac{\scalar{\bsgamma^*}{\bse_s}}{2\alpha_s} \right\rvert \geq \frac{p_s}{2\bar\alpha}(1-\scalar{\bsgamma^*}{\bsone}) + \scalar{\bslambda^*}{\bse_s} + \frac{\scalar{\bsgamma^*}{\bse_s}}{2\alpha_s}  \right), \forall s \\
    0 &= \sum_{s=1}^K \left(\frac{p_s}{\bar\alpha}\bsone - \frac{\bse_s}{\alpha_s}\right) \Prob_{\bsX \mid S=s}\left(\min\left(2\eta(\bsX, S), \eta(\bsX, S) - \frac{\baralpha \lambda_s}{p_s}\right) \geq \frac{\baralpha \gamma_s}{p_s \alpha_s} + 1 - \bar\gamma \right)\enspace.
\end{aligned}
\end{equation}

\paragraph{Feasibility of $\tilde{g}$ for the primal problem}
Let us check that $\tilde{g}$ is feasible for the primal problem. Using the definition of $\tilde{g}$ and the first-order optimal condition on $\bslambda^*$ we obtain, for any $s\in [K]$,
\begin{align*}
    \Prob(\tilde{g}(\bsX, S) \neq r \mid S=s) &= \Prob_{\bsX \mid S=s}\left( \left\lvert  \frac{p_s}{2\bar\alpha}(1-2\eta(\bsX,s)-\scalar{\bsgamma}{\bsone}) + \frac{\scalar{\bsgamma}{\bse_s}}{2\alpha_s} \right\rvert \geq \frac{p_s}{2\bar\alpha}(1-\scalar{\bsgamma}{\bsone}) + \scalar{\bslambda}{\bse_s} + \frac{\scalar{\bsgamma}{\bse_s}}{2\alpha_s} \right)\\
    &=\alpha_s\enspace,
\end{align*}
which proves that $\tilde{g}$ satisfies the first set of constraints. For the Demographic Parity constraints, one easily obtains
\begin{align*}
    \mathbb{P}_{\bsX\mid S=s}(\tilde{g}(\bsX, S)=1 \mid \tilde{g}(\bsX, S) \neq r) &= \frac{1}{\alpha_s} \Prob_{\bsX \mid S=s}(\tilde{g}(\bsX, S)=1)\\
    &= \frac{1}{\alpha_s} \Prob_{\bsX \mid S=s}\left(\min\left(2\eta(\bsX, S), \eta(\bsX, S) - \frac{\baralpha \lambda_s}{p_s}\right) \geq \frac{\alpha \gamma_s}{p_s \alpha_s} + 1 - \bar\gamma \right)\enspace,\\
    \Prob_{(\bsX, S)}(\tilde{g}(\bsX, S)=1 \mid \tilde{g}(\bsX, S) \neq r) &= \sum_{s=1}^K \frac{p_s}{\bar\alpha} \Prob_{\bsX \mid S=s}\left(\min\left(2\eta(\bsX, S), \eta(\bsX, S) - \frac{\baralpha \lambda_s}{p_s}\right) \geq \frac{\baralpha \gamma_s}{p_s \alpha_s} + 1 - \bar\gamma \right)\enspace.
\end{align*}
The first-order optimality condition for $\bsgamma^*$ guarantees that for, any $s \in [K]$,
\begin{align*}
\mathbb{P}_{\bsX\mid S=s}(\tilde{g}(\bsX, S)=1 \mid \tilde{g}(\bsX, S) \neq r) =  \Prob_{(\bsX, S)}(\tilde{g}(\bsX, S)=1 \mid \tilde{g}(\bsX, S) \neq r)\enspace, 
\end{align*}
\textit{i.e.} it guarantees that the classifier $\tilde{g}$ satisfies the Demographic Parity constraint.

We conclude that the classifier $\tilde{g}$ is feasible for the primal problem and thus that strong duality holds.

\section{Auxiliary results}
\label{app:aux_results}
We will need a tight control on the sup-norm of the difference between CDF and empirical CDF. The next result is~\citep[Corollary 1]{massart1990tight}.
\begin{theorem}
\label{thm:DKW}
Let $\bsZ, \bsZ_1, \ldots, \bsZ_n$ be $n + 1$ \iid continuous random variable sampled from $\Prob$ on $\class{Z}$, then for any $\delta > 0$, with probability at least $1 - \delta$,
\begin{align*}
    \sup_{z \in \bbR}\abs{\frac{1}{n}\sum_{i = 1}^n \ind{Z_i \leq z} - \Prob(Z \leq z)} \leq \sqrt{\frac{\log(2/\delta)}{2n}}\enspace.
\end{align*}
\end{theorem}

\section{Control of reject rate}
\label{app:control_reject}

\begin{proposition}
\label{prop:reject}
For all $\delta \in (0, 1)$, the proposed algorithm satisfies with probability at least $1 - \delta$ that

\begin{align*}
    \left\lvert \Prob(\hat{g}(\bsX, S) \neq r\mid S = s) - \alpha_s \right\rvert \leq \sqrt{\frac{2\log(2K / \delta)}{n_s}} + \frac{2}{n_s},\quad\forall s \in [K]\enspace.
\end{align*}

\end{proposition}

The rest of this section is devoted to the proof of this result. In what follows, all the derivations should be understood conditionally on $\heta$. In simple words, the estimator $\heta$ is treated as fixed and the only randomness comes from the unlabeled data.
According to the definition of our estimator,  
\begin{align*}
    \Prob_{\bsX | S = s}\left(\hat{g}(\bsX, s) \neq r\right) = \Prob_{\bsX \mid S=s}\left(\hat{G}(\bsX, s, \hbslambda, \hat{
    \bsgamma}) > 0 \right)\enspace.
\end{align*}

Using the triangle inequality we can upper bound $\lvert \CPs{\hat{G}(\bsX, s, \hbslambda, \hbsgamma) > 0} - \alpha_s \rvert$ by two terms
\begin{equation}
\label{eq:reject_0}
\begin{aligned}
     \underbrace{\left\lvert \CPs{\hat{G}(\bsX, s, \hbslambda, \hbsgamma) > 0} - \ECPs{\hat{G}(\bsX, s, \hbslambda, \hbsgamma) > 0} \right\rvert}_{\texttt{T}_1}
    + \underbrace{\left\lvert \ECPs{\hat{G}(\bsX, s, \hbslambda, \hbsgamma) > 0} - \alpha_s \right\rvert}_{\texttt{T}_2}\enspace,
\end{aligned}
\end{equation}
which are treated separately.
\paragraph{Control of $\texttt{T}_1$.}
The first term $\texttt{T}_1$ can be controlled using tools from empirical process theory.
One can directly observe that
\begin{equation}
    \label{eq:reject_1}
\begin{aligned}
    \texttt{T}_1
    &\leq
    \sup_{(\bslambda, \bsgamma) \in \bbR^K \times \bbR^K}
    \left\lvert \CPs{\hat{G}(\bsX, s, {\bslambda}, {\bsgamma}) > 0} - \ECPs{\hat{G}(\bsX, s, {\bslambda}, {\bsgamma}) > 0} \right\rvert\\
    &\leq
    \sup_{(a, b) \in \bbR\times\bbR} 
    \left\lvert \CPs{\abs{\frac{p_s}{2\bar\alpha}\heta(\bsX, S) - a} - a + b > 0} - \ECPs{\abs{\frac{p_s}{2\bar\alpha}\heta(\bsX, S) - a} - a + b > 0} \right\rvert\\
    &\leq
    \sup_{(a, c) \in \bbR\times\bbR} 
    \left\lvert \CPs{\abs{\frac{p_s}{2\bar\alpha}\heta(\bsX, S) - a} > c} - \ECPs{\abs{\frac{p_s}{2\bar\alpha}\heta(\bsX, S) - a} > c} \right\rvert\\
    &\leq 2\sup_{a \in \bbR}\abs{\Prob_{\bsX|S=s}(\heta(\bsX, S) \leq a) - \hProb_{\bsX|S=s}(\heta(\bsX, S) \leq a)}\enspace,
\end{aligned}
\end{equation}
where we used the triangle inequality and the fact that $(\hat{\eta}(\bsX, S) \mid S=s)$ is a continuous random variable to obtain the last inequality.

By our assumption (see Remark~\ref{rem:cont}), the random variables $\heta(\bsX_i, s), (\heta(\bsX, S) \mid S = s)$ for $i \in \class{I}_s$ are \iid continuous conditionally on $\heta$. Thus, applying Theorem~\ref{thm:DKW} we conclude that with probability at least $1 - \delta$ it holds that
\begin{align}
    \label{eq:T1}
    \texttt{T}_1 \leq \sqrt{\frac{2\log(2 / \delta)}{n_s}}\enspace.
\end{align}


\paragraph{Control of $\texttt{T}_2$.}
The control of the second term $\texttt{T}_2$ requires a more involved analysis.
Since $\hbslambda$ is a minimizer of~\eqref{eq:emp_min}, the first order optimality condition for convex non-smooth minimization problems state that for any $s \in [K]$, there exists $\rho_s \in [0, 1]$ such that
\begin{align*}
    \alpha_s = \ECPs{\hat{G}(\bsX, s, \hbslambda, \hbsgamma) > 0} + \rho_s \ECPs{\hat{G}(\bsX, s, \hbslambda, \hbsgamma) = 0} 
\end{align*}

Thus, the second term of Eq.~\eqref{eq:reject_0} can be bounded as
\begin{align}
\label{eq:empirical_reject_control}
    \left\lvert \ECPs{\hat{G}(\bsX, s, \hbslambda, \hbsgamma) > 0} - \alpha_s \right\rvert \leq \ECPs{\hat{G}(\bsX, s, \hbslambda, \hbsgamma) = 0} \enspace.
\end{align}
The control of $\ECPs{\hat{G}(\bsX, s, \hbslambda, \hbsgamma) = 0}$ is provided by the following result.
\begin{lemma}
\label{lem:pigh}
Assume that $(\hat\eta(\bsX, S) \mid S = s, \heta)$ is almost surely continuous, then for any $s \in [K]$, for any $(\bslambda, \bsgamma)$,
\begin{align*}
    \ECPs{\hat{G}(\bsX, s, \bslambda, \bsgamma) = 0}  \leq \frac{2}{n_s}\enspace,\qquad \text{a.s.}
\end{align*}
\end{lemma}

\begin{proof}
We recall that by definition of $\hat\Prob_{\bsX|s}$ we have
\begin{align*}
    \ECPs{\hat{G}(\bsX, s, \bslambda, \bsgamma) = 0} = \frac{1}{n_s}\sum_{i=1}^{n_s} \mathds{1}(\hat{G}(\bsX_i, s, \bslambda, \bsgamma) = 0)\enspace.
\end{align*}
The proof goes by contradiction. Assume that the event
\begin{align*}
    \frac{1}{n_s}\sum_{i=1}^{n_s} \mathds{1}(\hat{G}(\bsX_i, s, \bslambda, \bsgamma) = 0) \geq \frac{3}{n_s}\enspace,
\end{align*}
happens with positive probability. Then, there exist three indexes $i_1, i_2, i_3$ such that
\begin{align*}
    \hat{G}(\bsX_{i_j}, s, \bslambda, \bsgamma) = 0\enspace, \quad j=1, 2, 3\enspace.
\end{align*}
However, $\hat{G}(\bsX, s, \bslambda, \bsgamma) = 0$ implies that either
\begin{align*}
    \frac{\hat{p}_s}{\bar\alpha}\hat{\eta}(\bsX, s) + \scalar{\bslambda}{\bse_s} = 0
    \quad\text{ or }\quad
    \frac{\hat{p}_s}{\bar\alpha} (\hat{\eta}(\bsX, s) + \scalar{\bsgamma}{1} - 1) - \scalar{\bslambda}{\bse_s} + \frac{\scalar{\bsgamma}{\bse_s}}{\alpha_s}=0\enspace.
\end{align*}

By the pigeonhole principle, there exist $i, j \in \{i_1, i_2, i_3\}, i\neq j$ such that
\begin{align*}
    \hat{\eta}(\bsX_{i}, s) = \hat{\eta}(\bsX_j, s)\enspace,
\end{align*}
which contradicts our assumption that $(\hat{\eta}(\bsX, S) \mid S = s, \heta)$ is continuous almost surely.
\begin{remark}
\label{rem:cont}
Recall that the assumption of continuity of $(\hat{\eta}(\bsX, S) \mid S = s, \heta)$ can always be fulfilled with the help of additional randomization. More formally, one needs to replace $\heta$ by its smoothed version using additional randomization present in Algorithm~\ref{alg:1}. To keep things simple, we avoid this technicality in our proof and simply assume that $(\hat{\eta}(\bsX, S) \mid S = s, \heta)$ is indeed continuous. The statement of this result is straightforwardly adapted to the perturbed version of $\heta$.
\end{remark}
\end{proof}
Lemma~\ref{lem:pigh} allows to control the second term in Eq.~\eqref{eq:reject_0} yielding
\begin{align}
    \label{eq:T2}
    \texttt{T}_2 \leq \frac{2}{n_s}\enspace.
\end{align}



\paragraph{Putting together.}

Substituting Eqs.~\eqref{eq:T1} and~\eqref{eq:T2} into Eq.~\eqref{eq:reject_1}, we deduce that for all $s \in [K]$ we have, with probability $1 - \delta$, 
\begin{align*}
    \left\lvert \CPs{\hat{G}(\bsX, s, \hbslambda, \hbsgamma) > 0} - \alpha_s \right\rvert
    \leq
    \sqrt{\frac{2\log(2 / \delta)}{n_s}}
    +
    \frac{2}{n_s}\enspace.
\end{align*}
Finally, taking the union bound we deduce that, with probability at least $1 - \delta$, we have for all $s\in[K]$ 
\begin{align*}
    \left\lvert \CPs{\hat{G}(\bsX, s, \hbslambda, \hbsgamma) > 0} - \alpha_s \right\rvert
    \leq
    \sqrt{\frac{2\log(2K / \delta)}{n_s}}
    +
    \frac{2}{n_s}\enspace.
\end{align*}

The proof of Proposition~\ref{prop:reject} is concluded.

\section{Control of Demographic Parity violation}
\label{app:control_dp}
\begin{proposition}
\label{prop:control_dp}
For any $\delta \in (0, 1)$, the proposed algorithm satisfies with probability at least $1 - \delta$, for any $s \in [K]$,
\begin{align*}
    \abs{\CPs{\hat{g}(\bsX, s)=1 \mid \hat{g}(\bsX, s) \neq r} - \mathbb{P}_{(\bsX, S)}\left( \hat{g}(\bsX, S)=1 \mid \hat{g}(\bsX, S) \neq r \right)} \leq \frac{1}{\alpha_s}v_{n_s}^{\delta, K} + \frac{1}{\bar\alpha}\sum_{s = 1}^Kp_s v_{n_s}^{\delta, K}\enspace,
\end{align*}
where 
\begin{align*}
    v_n^{\delta, K} \eqdef \left(3 \sqrt\frac{\log(\sfrac{4K}{\delta})}{n} + \frac{4}{n}\right)\enspace.
\end{align*}
\end{proposition}
\begin{remark}
It is easy to see in the proof that the high-probability event on which Proposition~\ref{prop:control_dp} holds is contained in the high-probability event on which Proposition~\ref{prop:reject} holds.
\end{remark}
The rest of this section is devoted to the proof of this result.\\

\paragraph{Problem splitting.} Similarly to the control of the reject rate we start by splitting our problem in several parts.
Recall that our goal here is to control
\begin{align*}
    (\text{DP}^s) &\eqdef \abs{\CPs{\hat{g}(\bsX, s)=1 \mid \hat{g}(\bsX, s) \neq r} - \mathbb{P}_{(\bsX, S)}\left( \hat{g}(\bsX, S)=1 \mid \hat{g}(\bsX, S) \neq r \right)}\enspace,
\end{align*}
for all $s \in [K]$.
Triangle inequality yields that
\begin{align*}
    (\text{DP}^s) \leq &\abs{\CPs{\hat{g}(\bsX, s)=1 \mid \hat{g}(\bsX, s) \neq r} - \alpha_s^{-1}\CPs{\hat{g}(\bsX, s)=1}}\\
    &+ \abs{\alpha_s^{-1}\CPs{\hat{g}(\bsX, s)=1} - \alpha_s^{-1}\ECPs{\hat{g}(\bsX, s)=1}}\\
    &+ \abs{\alpha_s^{-1}\ECPs{\hat{g}(\bsX, s)=1} - \baralpha^{-1} \sum_{s  \in [K]} p_s \ECPs{\hat{g}(\bsX, s)=1}}\\
    &+ \abs{\baralpha^{-1} \sum_{s  \in [K]} p_s \ECPs{\hat{g}(\bsX, s)=1} - \baralpha^{-1} \sum_{s  \in [K]} p_s \CPs{\hat{g}(\bsX, s)=1}}\\
    &+ \abs{\baralpha^{-1} \sum_{s  \in [K]} p_s \CPs{\hat{g}(\bsX, s)=1} - \mathbb{P}_{(\bsX, S)}\left( \hat{g}(\bsX, S)=1 \mid \hat{g}(\bsX, S) \neq r \right)}\enspace.
\end{align*}

The second and the fourth terms will be controlled using empirical process theory. We can get a bound on the first and fifth terms through our control of the reject rate. The third term is controlled via the first-order optimality condition on $\bsgamma$.

\paragraph{High-probability event.} Let us describe in details the high-probability event on which we will place ourselves for controlling all the terms, uniformly over the classes $s \in [K]$.

Proposition~\ref{prop:reject} states that there exists an event $\mathbf{R}$ that holds with probability at least $1-K\delta$ and on which, for any $\delta \in (0, 1/K)$, the proposed algorithm satisfies with probability at least $1 - K\delta$ that
\begin{align*}
    \left\lvert \Prob(\hat{g}(\bsX, S) \neq r\mid S = s) - \alpha_s \right\rvert \leq u_{n_s}^\delta,\quad\forall s \in [K]\enspace,
\end{align*}
where

\begin{align*}
    u^{\delta}_{n} \eqdef \sqrt\frac{2\log(\sfrac{2}{\delta})}{n} + \frac{2}{n}\enspace, \quad \forall n \geq 1\enspace.
\end{align*}

Furthermore, for any class $s\in[K]$, using the fact that the random variable $(\eta(\bsX, S)\mid S =s)$ is continuous, the event
\begin{align*}
    \text{EP}_s \coloneqq \left\{ \sup_{a \in \mathbb{R}} \abs{\CPs{\eta(\bsX, s) > a} - \ECPs{\eta(\bsX, s) > a}} \leq \sqrt{\frac{\log(\sfrac{2}{\delta})}{2n_s}} \right\}\enspace,
\end{align*}
holds with probability at least $1 - \delta$ (see Theorem~\ref{thm:DKW}).
By a simple union bound argument, the intersection of those events, denoted by $\text{EP} \eqdef \cap_{s \in [K]} \text{EP}_s$, then holds with probability at least $1 - 2K\delta$.

In what follows we place ourselves on the event $\mathbf{A} \eqdef \mathbf{R} \cap \text{EP}$ which holds with probability at least $1-2K\delta$.

\paragraph{First-order optimality condition for $\hbsgamma$.} Recall that $(\hbslambda, \hbsgamma)$ is a solution of
\begin{align*}
     \min_{(\bslambda, \bsgamma)} \left[\scalar{\bslambda}{\bsalpha} + \hat{\Exp}_{\bsX \mid S = s} 
     (\hat{G}(\bsX, s,
     \bslambda,
     \bsgamma ))_+\right]\enspace,
\end{align*}
where the function $\hat{G}$ is defined as \begin{align*}
    \hat{G}(\bsx, s, \bslambda, \bsgamma) =
    \left\lvert  \frac{\hat{p}_s}{2\bar\alpha}(1-2\hat\eta(\bsx,s)-\scalar{\bsgamma}{\bsone}) + \frac{\scalar{\bsgamma}{\bse_s}}{2\alpha_s} \right\rvert  
    - \frac{\hat{p}_s}{2\bar\alpha}(1-\scalar{\bsgamma}{\bsone}) - \scalar{\bslambda}{\bse_s} - \frac{\scalar{\bsgamma}{\bse_s}}{2\alpha_s}\enspace.
\end{align*}

The positive part of $\hat{G}$ can be expressed as
\begin{align*}
    (\hat{G}(\bsx, s, \bslambda, \bsgamma))_+ = \max(0, m_+(\bsx, s, \bslambda, \bsgamma), m_-(\bsx, s, \bslambda, \bsgamma))\enspace,
\end{align*}
where
\begin{align*}
    m_+(\bsx, s, \bslambda, \bsgamma) = -\frac{p_s}{\bar\alpha}\hat\eta(\bsx,s) - \lambda_s \quad\text{ and }\quad
    m_-(\bsx, s, \bslambda, \bsgamma) = \frac{p_s}{\bar\alpha}(\hat{\eta}(\bsx,s) + \scalar{\bsgamma}{1} - 1) - \frac{\scalar{\bsgamma}{\bse_s}}{\alpha_s} - \lambda_s\enspace.
\end{align*}
Noticing that the event $m_-(\bsx, s, \bslambda, \bsgamma) > \max(0, m_+(\bsx, s, \bslambda, \bsgamma))$ is the same as the event $\hat{g}(\bsX, s) = 1$,
the first-order optimality condition on $\hbsgamma$ reads as
\begin{align*}
    \exists (\rho_s)_{s=1}^K \in [0, 1]^K \text{ s.t. } \sum_{s=1}^K \left(\frac{p_s}{\bar\alpha}\bsone - \frac{1}{\alpha_s}\bse_s\right) \left( \ECPs{\hat{g}(\bsX, s) = 1} + \rho_s \ECPs{\Delta_s(\hbslambda,\hbsgamma)}\right) = 0\enspace,
\end{align*}
where we define the event $\Delta_s(\bslambda,\bsgamma)\eqdef\{m_-(\bsX, s, \bslambda, \bsgamma) = \max(0, m_+(\bsX, s, \bslambda, \bsgamma))\}$.
In scalar form the previous condition can be express as: for any $s \in [K]$, there exists $\rho_s \in [0,1]$ such that
\begin{align*}
    \sum_{s=1}^K \frac{p_s}{\bar\alpha}\left( \ECPs{\hat{g}(\bsX, s) = 1} + \rho_s \ECPs{\Delta_s(\hbslambda,\hbsgamma)}\right) = \frac{1}{\alpha_s} \left( \ECPs{\hat{g}(\bsX, s) = 1} + \rho_s \ECPs{\Delta_s(\hbslambda,\hbsgamma)}\right)\enspace.
\end{align*}

\paragraph{Control of the first term.}
Re-arranging terms and using the fact that $\CPs{\hat{g}(\bsX, S)=1} \leq \CPs{\hat{g}(\bsX, S)\neq r}$,
\begin{align*}
    (\text{DP}_1^s) &\eqdef \abs{\CPs{\hat{g}(\bsX, s)=1 \mid \hat{g}(\bsX, s) \neq r} - \alpha_s^{-1}\CPs{\hat{g}(\bsX, S)=1}}\\
    &= \abs{\frac{1}{\alpha_s} - \frac{1}{\CPs{\hat{g}(\bsX, s)\neq r}}} \CPs{\hat{g}(\bsX, S)=1}\\
    &\leq \abs{\frac{1}{\alpha_s} - \frac{1}{\CPs{\hat{g}(\bsX, s)\neq r}}} \CPs{\hat{g}(\bsX, S)\neq r}\\
    &= \frac{1}{\alpha_s}\abs{\CPs{\hat{g}(\bsX, S)\neq r}- \alpha_s}\enspace.
\end{align*}

Considering that we restrict ourselves to the high-probability event $\mathbf{A}$, we can conclude that
\begin{align*}
    (\text{DP}_1^s) \leq \frac{u_{n_s}^{\delta}}{\alpha_s}\enspace.
\end{align*}

\paragraph{Control of the second term} The second term is given by the empirical process
\begin{align*}
    (\text{DP}_2^s) &\eqdef \alpha_s^{-1} \abs{\CPs{\hat{g}(\bsX, S)=1} - \ECPs{\hat{g}(\bsX, S)=1}}\enspace.
\end{align*}

The  event $\{ \hat{g}(\bsX, s) = 1\}$ is the same as the event
\begin{align*}
    \left\{ \abs{\frac{p_s}{2\bar\alpha}(1-2\hat{\eta}(\bsX, s) - \scalar{\bsone}{\hbsgamma}) + \frac{\hbsgamma_s}{2\alpha_s}} > \frac{p_s}{2\bar\alpha}(1 - \scalar{\bsone}{\bsgamma}) + \hbslambda_s + \frac{\hbsgamma_s}{2\alpha_s}, \quad 2\hat{\eta}(\bsX, s) \geq 1 + \frac{\baralpha\hbsgamma_s}{\alpha_s p_s} - \scalar{\hbsgamma}{\bsone}  \right\}\enspace,
\end{align*}
which can be compacted to
\begin{align*}
    S(\bslambda, \bsgamma) \eqdef \left\{ \hat{\eta}(\bsX, s) > \max\left(\frac{1}{2} + \frac{\baralpha\hbsgamma_s}{2\alpha_s p_s} - \frac{1}{2}\scalar{\hbsgamma}{\bsone}, \frac{\bar\alpha}{p_s}\left(\hbslambda_s + \frac{\hbsgamma_s}{\alpha_s}\right) + 1 - \scalar{\bsone}{\bsgamma}\right) \right\}\enspace.
\end{align*}
Following this observation, we can express the second term as
\begin{align*}
     (\text{DP}_2^s) &= \alpha_s^{-1} \sup_{(\bslambda, \bsgamma)} \abs{\CPs{S(\bslambda, \bsgamma)} - \ECPs{S(\bslambda, \bsgamma)}}
     \leq \alpha_s^{-1}\sup_{a \in \mathbb{R}} \abs{\CPs{ \hat{\eta}(\bsX, s) > a} - \ECPs{\hat{\eta}(\bsX, s) > a}}\enspace.
\end{align*}
Since we are on the event $\mathbf{A}$ which is contained in the event $\text{EP}_s$, we have
\begin{align*}
    (\text{DP}_2^s) \leq \frac{1}{\alpha_s} \sqrt{\frac{\log(\sfrac{2}{\delta})}{2n_s}}\enspace.
\end{align*}

\paragraph{Control of the third term.}
The third term can be controlled with the first-order optimality condition on $\hbsgamma$ and multiple triangle inequalities as
\begin{align*}
    (\text{DP}_3^s) &\eqdef \abs{\alpha_s^{-1}\ECPs{\hat{g}(\bsX, S)=1} - \baralpha^{-1} \sum_{s  \in [K]} p_s \ECPs{\hat{g}(\bsX, S)=1}}\\
    &= \left\lvert \frac{\rho_s}{\alpha_s} \ECPs{\Delta_s(\hbslambda, \hbsgamma)} - \sum_{s = 1}^K \frac{p_s}{\bar\alpha} \rho_s \ECPs{\Delta_s(\hbslambda, \hbsgamma)} \right\rvert\\
    &\leq \frac{1}{\alpha_s} \ECPs{\Delta_s(\hbslambda, \hbsgamma)} +  \sum_{s = 1}^K \frac{p_s}{\bar\alpha} \ECPs{\Delta_s(\hbslambda, \hbsgamma)}\enspace.
\end{align*}

The following lemma gives an almost sure upper bound on $\ECPs{\Delta_s(\hbslambda, \hbsgamma)}$ for any $s\in [K]$.

\begin{lemma}
\label{lem:pigh2}
Assume that $(\hat\eta(\bsX, S) \mid S = s, \heta)$ is almost surely continuous, then for any $s \in [K]$, for any $(\bslambda, \bsgamma)$,
\begin{align*}
    \ECPs{\Delta_s(\hbslambda, \hbsgamma)}  \leq \frac{2}{n_s},\qquad \text{a.s.}
\end{align*}
\end{lemma}

\begin{proof}
This proof is similar to proof of Lemma~\ref{lem:pigh}. Assume by contradiction that the stated bound is not true. Then, it happens with positive probability that
\begin{align*}
    \frac{1}{n_s} \sum_{i=1}^{n_s} \mathds{1}\left\{m_-(\bsX_i, s, \bslambda, \bsgamma) = \max(0, m_+(\bsX_i, s, \bslambda, \bsgamma))\right\}\geq \frac{3}{n_s}\enspace,
\end{align*}
which implies that there exist a triplet $i_1,i_2,i_3$ such that
\begin{align*}
    m_-(\bsX_{i_j}, s, \bslambda, \bsgamma) = \max(0, m_+(\bsX_{i_j}, s, \bslambda, \bsgamma)), \quad \text{for } j=1,2,3\enspace. 
\end{align*}
By the pigeonhole principle, there must exist a couple $(i,j), i\neq j$ among this triplet such that either
\begin{align*}
     m_-(\bsX_{i}, s, \bslambda, \bsgamma) = m_-(\bsX_{j}, s, \bslambda, \bsgamma)
\end{align*}
or\begin{align*}
    m_-(\bsX_{i}, s, \bslambda, \bsgamma) - m_+(\bsX_{i}, s, \bslambda, \bsgamma) = m_-(\bsX_{j}, s, \bslambda, \bsgamma) - m_+(\bsX_{j}, s, \bslambda, \bsgamma)\enspace.
\end{align*}
In both cases one must have $\hat{\eta}(\bsX_i, s) = \hat{\eta}(\bsX_j, s)$ which happens with probability $0$ by the continuity assumption and leads to a contradiction. The proof of lemma is concluded
\end{proof}

Plugging in the bounds from Lemma~\ref{lem:pigh2} yields  
\begin{align*}
    (\text{DP}_3^s) \leq \frac{2}{n_s \alpha_s} + \frac{2}{\bar\alpha}\sum_{s=1}^K \frac{p_s}{n_s}\enspace.
\end{align*}

\paragraph{Control of the fourth term.}
The fourth term can be seen as a sum of empirical processes:
\begin{align*}
    (\text{DP}_4) &\eqdef \baralpha^{-1} \abs{ \sum_{s  \in [K]} p_s \ECPs{\hat{g}(\bsX, S)=1} - \sum_{s  \in [K]} p_s \CPs{\hat{g}(\bsX, S)=1}}\\
    &\leq \baralpha^{-1} \sum_{s = 1}^K p_s \abs{ \ECPs{\hat{g}(\bsX, S)=1} - \CPs{\hat{g}(\bsX, S)=1}}\enspace.
\end{align*}
We can control the fourth term from the bound we have on the second term (which holds uniformly over the classes $s$) as
\begin{align*}
    (\text{DP}_4) \leq \frac{1}{\bar\alpha}\sum_{s \in K} p_s \sqrt{\frac{\log(\sfrac{2}{\delta})}{2n_s}} \enspace.
\end{align*}

\paragraph{Control of the fifth term.}
Finally, the fifth term can be bounded using the same trick as for the first term.
\begin{align*}
    (\text{DP}_5) &\eqdef  \abs{\baralpha^{-1} \sum_{s  \in [K]} p_s \CPs{\hat{g}(\bsX, S)=1} - \mathbb{P}_{(\bsX, S)}\left( \hat{g}(\bsX, s)=1 \mid \hat{g}(\bsX, s) \neq r \right)}\\
    &= \abs{\frac{1}{\bar\alpha} - \frac{1}{\sum_{s = 1}^Kp_s\CPs{\hat{g}(\bsX, s)\neq r}}} \sum_{s = 1}^Kp_s \CPs{\hat{g}(\bsX, S)=1}\\
    &\leq \abs{\frac{1}{\bar\alpha} - \frac{1}{\sum_{s = 1}^Kp_s\CPs{\hat{g}(\bsX, s)\neq r}}}\sum_{s = 1}^Kp_s \CPs{\hat{g}(\bsX, S)\neq r}\\
    &=\frac{1}{\bar\alpha}\abs{\sum_{s = 1}^K p_s (\CPs{\hat{g}(\bsX, s) \neq r} - \alpha_s)}
    \leq \frac{1}{\bar\alpha}\sum_{s = 1}^Kp_s u_{n_s}^\delta.
\end{align*}

\paragraph{Summary.}
Putting everything together, we have shown that, on the event $\mathbf{A}$ which holds with probability at least $1-2K\delta$, we have, for any $s\in [K]$,
\begin{align*}
    (\text{DP}^s) \leq \frac{1}{\alpha_s}\left(3\sqrt\frac{\log(\sfrac{2}{\delta})}{2n_s} + \frac{4}{n_s}\right) + \frac{2}{\bar\alpha}\sum_{s = 1}^Kp_s\left(3\sqrt\frac{\log(\sfrac{2}{\delta})}{2n_s} + \frac{4}{n_s}\right)\enspace.
\end{align*}

\section{Control of the excess risk}
\label{app:control_risk}
Define the sequence
\begin{align*}
    u^{\delta, K}_{n} \eqdef \sqrt{\frac{2\log(\sfrac{4K}{\delta})}{n}} + \frac{2}{n}\enspace, \quad \forall n \geq 1\enspace.
\end{align*}
We state and prove slightly more precise bound then the one presented in the main body.
\begin{proposition}
Assume that $u_{n_s}^{\delta, K} < \alpha_s <1-\frac{2}{n_s}$ for any $s \in [K]$ and that Assumption~\ref{as:continuous_eta} holds.
Then, for any $\delta \in (0, 1)$, the excess risk of the post-processing classifier with abstention $\hat{g}$ defined in Eq~\eqref{eq:classif_with_abstention} satisfies, with probability at least $1-\delta$, 
\begin{align*}
    \excess(\hat g) \leq \left(\frac{1}{\bar\alpha} + \frac{1}{\baralpha - \sum_s p_s u_{n_s}^{\delta, K}}\right) \|\eta - \heta\|_1 + 6 \sum_{s = 1}^K  \parent{\frac{p_s}{\bar\alpha} + \frac{1}{\alpha_s}} u_{n_s}^{\delta, K} \enspace.
\end{align*}
\end{proposition}

A quick inspection of the proof shows that the high-probability event on which the stated bound holds is the same as the event on which Proposition~\ref{prop:control_dp} holds, which is contained in the event on which Proposition~\ref{prop:reject} holds. Thus we can control the excess risk and the violation of the constraints on the same high-probability event.

\begin{proof}
Since, using Assumption~\ref{as:continuous_eta} we have established strong duality, the following equality holds
\begin{align}
    \label{eq:risk_g_star}
    \risk(g^*) = \max_{(\bslambda, \bsgamma)} \left\{ \sum_{s=1}^K \Exp_{\bsX|S=s}\left(\frac{p_s}{2\bar\alpha}(1-\bar\gamma) + \lambda_s + \frac{\gamma_s}{2\alpha_s} - \left\lvert  \frac{p_s}{2\bar\alpha}(1-2\eta(\bsx, s)-\bar\gamma) + \frac{\gamma_s}{2\alpha_s} \right\rvert \right)_- - \sum_{s=1}^K \lambda_s \alpha_s\right\}\enspace.
\end{align}
Besides, we can control the risk of any classifier $g$ as
\begin{equation}
    \label{eq:risk_g}
\begin{aligned}
    \risk(g)
    &=
    \sum_{s=1}^K \frac{p_s}{\Prob(g(\bsX, S) \neq r)} \Exp_{\bsX | S=s}\left[(1-\eta(\bsX, s)\mathds{1}_{g(\bsX, s) = 1} + \eta(\bsX, s) \mathds{1}_{g(\bsX, s) = 0}  \right]\\
    &\leq
    \sum_{s=1}^K \frac{p_s}{\Prob(g(\bsX, S) \neq r)} \Exp_{\bsX | S=s}\left[(1-\heta(\bsX, s)\mathds{1}_{g(\bsX, s) = 1} + \heta(\bsX, s) \mathds{1}_{g(\bsX, s) = 0}  \right] + \frac{\|\eta - \heta\|_1}{\Prob(g(\bsX, S) \neq r)}\enspace.
\end{aligned}
\end{equation}
Setting $\texttt{A}_s(g) := \frac{p_s}{\bar\alpha}\Exp_{\bsX | S=s}\left[(1-\heta(\bsX, s))\mathds{1}_{g(\bsX, s) = 1} + \heta(\bsX, s) \mathds{1}_{g(\bsX, s) = 0}  \right]$, we have for any classifier $g$,
\begin{align*}
    \risk(g) \leq \sum_{s = 1}^K\texttt{A}_s(g) + \frac{\|\eta - \heta\|_1}{\Prob(g(\bsX, S) \neq r)} + \frac{1}{\bar\alpha}\abs{\Prob(g(\bsX, S) \neq r) - \bar\alpha}\enspace.
\end{align*}

In what follows we bound $\text{r}_1(g) := \sum_{s = 1}^K\texttt{A}_s(g)$. Re-arranging terms we trivially have
\begin{align*}
    \text{r}_1(g) &=
    \sum_{s=1}^K \frac{p_s}{\bar\alpha} \Exp_{\bsX | S=s}\left[(1-\heta(\bsX, S))\mathds{1}_{g(\bsX, S) = 1} + \heta(\bsX, S) \mathds{1}_{g(\bsX, S) = 0}  \right]
    \pm
    \sum_{s=1}^K \hat\lambda_s \left\{ \Exp_{\bsX | S=s}\left[\mathds{1}_{g(\bsX, S) = 1} + \mathds{1}_{g(\bsX, S) = 0} \right] - \alpha_s  \right\}\\
    &\pm
    \sum_{s=1}^K \frac{\hat\gamma_s}{\alpha_s} \Exp_{\bsX | S=s}\left[\mathds{1}_{g(\bsX, S) = 1}\right] - \left(\sum_{s'=1}^K \hat\gamma_{s'}\right) \left(\sum_{s=1}^K \frac{p_s}{\bar\alpha} \Exp_{\bsX | S=s}\left[\mathds{1}_{g(\bsX, S) = 1}\right]\right)\\
    &=
    \sum_{s=1}^K \Exp_{\bsX | S=s}\left[\hat H_{(\bsX, s)}(g, \hbslambda, \hbsgamma)  \right] - \sum_{s=1}^K \hat\lambda_s \alpha_s
     -
    \sum_{s=1}^K \hat\lambda_s \left\{ \Exp_{\bsX | S=s}\left[\mathds{1}_{g(\bsX, S) = 1} + \mathds{1}_{g(\bsX, S) = 0} \right] - \alpha_s  \right\}\\
    &-
    \sum_{s=1}^K \frac{\hat\gamma_s}{\alpha_s} \Exp_{\bsX | S=s}\left[\mathds{1}_{g(\bsX, S) = 1}\right] - \left(\sum_{s'=1}^K \hat\gamma_{s'}\right) \left(\sum_{s=1}^K \frac{p_s}{\bar\alpha} \Exp_{\bsX | S=s}\left[\mathds{1}_{g(\bsX, S) = 1}\right]\right)\enspace,
\end{align*}
where
\begin{align*}
    \hat H_{(\bsx,s)}(g, \bslambda, \bsgamma) = 
    \begin{cases}
       0, &\text{ if } g(\bsx,s)=r\\
       \frac{p_s}{\bar\alpha} \heta(\bsx, s) + \lambda_s, &\text{ if } g(\bsx,s)=0\\
       \frac{p_s}{\bar\alpha}(1-\heta(\bsx,s) - \bar{\gamma}) + \lambda_s + \frac{\gamma_s}{\alpha_s}, &\text{ if } g(\bsx,s)=1
    \end{cases}\enspace,
\end{align*}
with $\bar\gamma = \sum_{s = 1}^K \gamma_s$.
Note that, by the definition of $\hat g$, it holds that
\begin{align*}
    \sum_{s=1}^K \Exp_{\bsX | S=s}\left[\hat H_{(\bsX, s)}(\hat g, \hbslambda, \hbsgamma)  \right] = \Exp(-\hat{G}(\bsX, s, \hbslambda, \hbsgamma))_-\enspace.
\end{align*}
Thus, it holds that
\begin{equation}
    \label{eq:r1_hat_g}
\begin{aligned}
    \text{r}_1(\hat g)
    &=
    \sum_{s=1}^K \Exp_{\bsX|S=s}\left(\frac{p_s}{2\bar\alpha}(1-\bar{\hat{\gamma}}) + \hat\lambda_s + \frac{\hat\gamma_s}{2\alpha_s} - \left\lvert  \frac{p_s}{2\bar\alpha}(1-2\heta(\bsX, s)-\bar{\hat{\gamma}}) + \frac{\hat\gamma_s}{2\alpha_s} \right\rvert \right)_- - \sum_{s=1}^K \hat\lambda_s \alpha_s\\
    &-\sum_{s = 1}^K\hat\lambda_s\parent{\Prob(\hat g(\bsX, S) \neq r \mid S = s) - \alpha_s}\\
    &-\sum_{s = 1}^K\hat\gamma_s\parent{\frac{\Prob(\hat g(\bsX, S) = 1 \mid S = s)}{\alpha_s} - \sum_{s' = 1}^K\frac{p_{s'}}{\bar\alpha}\Prob(\hat g(\bsX, S) = 1 \mid S = s')}\enspace.
\end{aligned}
\end{equation}
Finally, substituting Eq.~\eqref{eq:r1_hat_g} into Eq.~\eqref{eq:risk_g} we obtain the following upper bound on $\risk(\hat g)$
\begin{align*}
    \risk(\hat g)
    &\leq
    \sum_{s=1}^K \Exp_{\bsX|S=s}\left(\frac{p_s}{2\bar\alpha}(1-\bar{\hat{\gamma}}) + \hat\lambda_s + \frac{\hat\gamma_s}{2\alpha_s} - \left\lvert  \frac{p_s}{2\bar\alpha}(1-2\heta(\bsX, s)-\bar{\hat{\gamma}}) + \frac{\hat\gamma_s}{2\alpha_s} \right\rvert \right)_- - \sum_{s=1}^K \hat\lambda_s \alpha_s\\
    &-\sum_{s = 1}^K\hat\lambda_s\parent{\Prob(\hat g(\bsX, S) \neq r \mid S = s) - \alpha_s}
    -\sum_{s = 1}^K\hat\gamma_s\parent{\frac{\Prob(\hat g(\bsX, S) = 1 \mid S = s)}{\alpha_s} - \sum_{s' = 1}^K\frac{p_{s'}}{\bar\alpha}\Prob(\hat g(\bsX, S) = 1 \mid S = s')}\\
    &+ \frac{\|\eta - \heta\|_1}{\Prob(\hat{g}(\bsX, S) \neq r)} + \frac{1}{\bar\alpha}\abs{\Prob(\hat{g}(\bsX, S) \neq r) - \bar\alpha}\enspace,
\end{align*}
which holds almost surely.

Define the excess risk $\excess(\hat g) := \risk(\hat g) - \risk(g^*)$.
Note that, using the fact that mapping $x\mapsto (x)_-$ is $1$-Lipschitz followed by the triangle inequality, the difference
\begin{small}
\begin{align*}
    \abs{\left(\frac{p_s}{2\bar\alpha}(1-\bar\gamma) + \lambda_s + \frac{\gamma_s}{2\alpha_s} - \left\lvert  \frac{p_s}{2\bar\alpha}(1-2\hat{\eta}(\bsx, s)-\bar\gamma) + \frac{\gamma_s}{2\alpha_s} \right\rvert \right)_- - \left(\frac{p_s}{2\bar\alpha}(1-\bar\gamma) + \lambda_s + \frac{\gamma_s}{2\alpha_s} - \left\lvert  \frac{p_s}{2\bar\alpha}(1-2\eta(\bsx, s)-\bar\gamma) + \frac{\gamma_s}{2\alpha_s} \right\rvert \right)_-}\enspace,
\end{align*}
\end{small}
can be upper bounded by $\frac{p_s}{\bar\alpha}\abs{\hat{\eta}(\bsx,s) - \eta(\bsx, s)}$, for any $(\bsx, s, \bslambda, \bsgamma)$. Thus, replacing $(\bslambda^*, \bsgamma^*)$ by $(\hbslambda, \hbsgamma)$ in the expression for $\risk(g^*)$ in Eq.~\eqref{eq:risk_g_star} we obtain
\begin{equation}
    \label{eq:excess_1}
\begin{aligned}
    \excess(\hat g)
    \leq
    &\frac{\|\eta - \heta\|_1}{\bar\alpha} + \frac{1}{\bar\alpha}\abs{\Prob(g(\bsX, S) \neq r) - \bar\alpha} + \frac{\|\eta - \heta\|_1}{\Prob(g(\bsX, S) \neq r)}
    +\sum_{s = 1}^K|\hat\lambda_s|\abs{\Prob(\hat g(\bsX, S) \neq r \mid S = s) - \alpha_s}\\
    &+\sum_{s = 1}^K|\hat\gamma_s|\abs{\frac{\Prob(\hat g(\bsX, S) = 1 \mid S = s)}{\alpha_s} - \sum_{s' = 1}^K\frac{p_{s'}}{\bar\alpha}\Prob(\hat g(\bsX, S) = 1 \mid S = s')}\enspace.
\end{aligned}
\end{equation}
In the above inequality we can control all the terms.

Indeed, using the fact that on the event of Proposition~\ref{prop:control_dp} we have, with probability at least $1-2K\delta$,
\begin{align*}
    \abs{\Prob(\hat g(\bsX, S) \neq r \mid S = s) - \alpha_s} \leq u_{n_s}^\delta, \quad \forall s \in [K], \quad \text{with}\quad u^{\delta}_{n} \eqdef \sqrt{\frac{2\log(2/ \delta)}{n}} + \frac{2}{n}\enspace, \quad \forall n \geq 1\enspace,
\end{align*}

we deduce that with probability at least $1-2K\delta$ the following three inequalities hold
\begin{equation}
    \label{eq:excess_2}
\begin{aligned}
     &\frac{1}{\bar\alpha}\abs{\Prob(g(\bsX, S) \neq r) - \bar\alpha} \leq \frac{1}{\bar\alpha} \sum_{s = 1}^K p_s u_{n_s}^\delta\enspace,\\
     &\frac{\|\eta - \heta\|_1}{\Prob(\hat{g}(\bsX, S) \neq r)} \leq \frac{\|\eta - \heta\|_1}{\baralpha - \sum_s p_s u_{n_s}^\delta} \enspace,\\
     &\sum_{s = 1}^K|\hat\lambda_s|\abs{\Prob(\hat g(\bsX, S) \neq r \mid S = s) - \alpha_s} \leq \sum_{s = 1}^K|\hat\lambda_s| u_{n_s}^\delta\enspace.
\end{aligned}
\end{equation}
Note that by the assumption of the proposition, the term $\baralpha - \sum_s p_s u_{n_s}^\delta > 0$.

Furthermore, on the same event, using the notations of the proof of Proposition~\ref{prop:control_dp}, we have for any $s \in [K]$
\begin{equation}
    \label{eq:excess_3}
\begin{aligned}
    \abs{\frac{\CPs{\hat g(\bsX, S) = 1}}{\alpha_s} - \sum_{s' = 1}^K\frac{p_{s'}}{\bar\alpha}\Prob_{\bsX \mid S=s'}(\hat g(\bsX, S) = 1)} \leq (\text{DP}_2^s) + (\text{DP}_3^s) + (\text{DP}_4) \leq \frac{1}{\alpha_s}v_{n_s}^\delta + \frac{2}{\bar\alpha}\sum_{s = 1}^Kp_s v_{n_s}^\delta \enspace.
\end{aligned}
\end{equation}
where $ v_n^\delta = \sqrt{\tfrac{\log(\sfrac{2}{\delta})}{2n}} + \tfrac{2}{n}$.
All in all, substituting Eqs.~\eqref{eq:excess_2} and~\eqref{eq:excess_3} into Eq.~\eqref{eq:excess_1} we deduce that
\begin{align*}
    \excess(\hat g) &\leq \left(\frac{1}{\bar\alpha} + \frac{1}{\baralpha - \sum_s p_s u_{n_s}^\delta}\right) \|\eta - \heta\|_1 + \sum_{s = 1}^K \parent{\frac{p_s}{\bar\alpha} + \abs{\hat{\lambda}_s}}u_{n_s}^\delta + \sum_{s = 1}^K \left(\frac{\lvert\hat\gamma_s \rvert}{\alpha_s} +\frac{2p_s}{\bar\alpha} (\sum_{s'} \lvert\hat\gamma_{s'}\rvert)\right) v_{n_s}^\delta\\
    &= \left(\frac{1}{\bar\alpha} + \frac{1}{\baralpha - \sum_s p_s u_{n_s}^\delta}\right) \|\eta - \heta\|_1 + \sum_{s = 1}^K \parent{\frac{2p_s}{\bar\alpha} + 2\abs{\hat{\lambda}_s} + \frac{\abs{\hat\gamma_s}}{\alpha_s} + \frac{2p_s}{\baralpha} (\sum_{s'} \abs{\hat\gamma_s}) }\sqrt{\frac{\log(\sfrac{1}{\delta})}{2n_s}} \\ &+ \sum_{s=1}^K \left(\frac{p_s}{\bar\alpha} + \absin{\hat{\lambda}_s} + \frac{\lvert\hat\gamma_s \rvert}{\alpha_s} +\frac{2p_s}{\bar\alpha} (\sum_{s'} \lvert\hat\gamma_{s'}\rvert)\right) \frac{2}{n_s}\enspace.
\end{align*}
In order to finish the proof it remains to provide a bound on $|\hat\lambda_s|$ and $|\hat\gamma_s|$. Proposition~\ref{prop:bound_param}, proven below, establishes this bound and yields
\begin{align*}
    \excess(\hat g) 
    &\leq \left(\frac{1}{\bar\alpha} + \frac{1}{\baralpha - \sum_s p_s u_{n_s}^\delta}\right) \|\eta - \heta\|_1 + \sum_{s = 1}^K \left[ \parent{\frac{4p_s}{\bar\alpha} + \frac{3}{\alpha_s}}\sqrt{\frac{2\log(\sfrac{2}{\delta})}{n_s}} + \left(\frac{6}{\alpha_s} + \frac{6 p_s}{\bar\alpha} \right) \frac{2}{n_s}\right]\\
    &\leq \left(\frac{1}{\bar\alpha} + \frac{1}{\baralpha - \sum_s p_s u_{n_s}^\delta}\right) \|\eta - \heta\|_1 + 6 \sum_{s = 1}^K  \parent{\frac{p_s}{\bar\alpha} + \frac{1}{\alpha_s}} u_{n_s}^\delta\enspace.\\
\end{align*}

the proof is concluded after the observation that thanks to our assumption we have $\baralpha - \sum_s p_s u_{n_s}^\delta \geq \baralpha / 2$.
\end{proof}

\paragraph{Boundedness of optimal parameters}

\begin{proposition}
\label{prop:bound_param}
The minimization problem in Eq.~\eqref{eq:min_true} admits a global minimizer $(\bslambda^*, \bsgamma^*)$ which satisfies
\begin{align*}
    \lVert \bsgamma^* \rVert_1 \leq 2 \quad\text{and}\quad \abs{\lambda_s^*} \leq \frac{p_s}{\bar\alpha} \vee \frac{\abs{\bsgamma_s^*}}{\alpha_s}\enspace.
\end{align*}
Furthermore, if for any $s$, $n_s > \frac{2}{\alpha_s \wedge (1-\alpha_s)}$ and $\hat\eta(\cdot,s) \in [0, 1]$, the same holds for Eq.~\eqref{eq:emp_min}, that is,
\begin{align*}
    \lVert \hbsgamma \rVert_1 \leq 2\quad\text{and}\quad \absin{\hat\lambda_s} \leq \frac{p_s}{\bar\alpha} \vee \frac{\abs{\hbsgamma_s}}{\alpha_s} \enspace.
\end{align*}
\end{proposition}

\begin{proof}
We denote the conditional expectation of $Y$ given $S=s$ by $\eta(s)$.
Denote by $H(\bslambda, \bsgamma)$ the objective function of the minimization problem in Eq.~\eqref{eq:min_true}.

\paragraph{Existence of global minizer.} Fix arbitrary $(\bslambda, \bsgamma) \in \bbR^K \times \bbR^K$ such that $\sum_{s=1}^K \gamma_s = 0$. 
Since the function $x \mapsto (|x| - b)_+$ is convex for any $b \in \bbR$ we can lower bound $H(\bslambda, \bsgamma) $ using Jensen's inequality as
\begin{align*}
    H(\bslambda, \bsgamma) &= \frac{1}{2} \sum_{s=1}^K \frac{1}{\alpha_s} \Exp_{\bsX \mid S=s} \left( \abs{\frac{\alpha_s p_s}{\bar\alpha}(1-2\eta(\bsX, s)) + \gamma_s} - \frac{p_s \alpha_s}{\bar\alpha} - 2\alpha_s \lambda_s - \gamma_s \right)_+ + \sum_{s = 1}^K\lambda_s \alpha_s\\
    &\geq \frac{1}{2} \sum_{s=1}^K \frac{1}{\alpha_s} \left(  \abs{\frac{\alpha_s p_s}{\bar\alpha}(1-2\eta(s)) + \gamma_s} - \frac{p_s \alpha_s}{\bar\alpha} - 2\alpha_s \lambda_s - \gamma_s \right)_+ + \sum_{s = 1}^K\lambda_s \alpha_s\enspace.
\end{align*}
Furthermore, since $\alpha_s \leq 1$ for any $s$ and by assumption, $\bar\gamma=0$, we can further lower bound $H(\bslambda, \bsgamma)$ as
\begin{align}
    H(\bslambda, \bsgamma) &\geq \frac{1}{2} \sum_{s=1}^K \left(  \abs{\frac{\alpha_s p_s}{\bar\alpha}(1-2\eta(s)) + \gamma_s} - \frac{p_s \alpha_s}{\bar\alpha} - 2\alpha_s \lambda_s - \gamma_s \right)_+ + \sum_{s = 1}^K\lambda_s \alpha_s\nonumber\\
    &\geq \frac{1}{2} \left(\lVert \bsgamma \rVert_1-\sum_{s=1}^K \frac{\alpha_s p_s}{\bar\alpha} \abs{(1-2\eta(s))} - 1 - 2 \sum_{s = 1}^K\lambda_s \alpha_s \right)_+ + \sum_{s = 1}^K\lambda_s \alpha_s\nonumber\\
    &\geq \frac{\lVert \bsgamma \rVert_1}{2} - 1\enspace,\label{eq:bound_gamma}
\end{align}
where we used the triangle inequality for the second inequality and we lower bounded the positive part by the number itself and  upper bounded $\abs{1-2\eta(s)}$ by one.

Besides, notice that
\begin{align*}
     H(\bslambda, \bsgamma)
     &=
     \frac{1}{2} \sum_{s=1}^K \frac{1}{\alpha_s} \Exp_{\bsX \mid S=s} \left( \abs{\frac{\alpha_s p_s}{\bar\alpha}(1-2\eta(\bsX, s)) + \gamma_s} - \frac{p_s \alpha_s}{\bar\alpha} - 2\alpha_s \lambda_s - \gamma_s \right)_+ + \sum_{s = 1}^K\lambda_s \alpha_s\\
     &\geq
     \sum_{s=1}^K \frac{1}{\alpha_s} \Exp_{\bsX \mid S=s} \left( -\frac{\alpha_s p_s}{\bar\alpha}\eta(\bsX, s) - \alpha_s \lambda_s\right)_+ + \sum_{s = 1}^K\lambda_s \alpha_s\\
     &\geq
     \sum_{s=1}^K \left\{\left( -\frac{p_s}{\bar\alpha}\eta(s) - \lambda_s\right)_+ + \lambda_s \alpha_s \right\}
\end{align*}
One easily observes that
\begin{align}
    \label{eq:bound_lambda}
    \sum_{s=1}^K \left\{\left( -\frac{p_s}{\bar\alpha}\eta(s) - \lambda_s\right)_+ + \lambda_s \alpha_s \right\}
    \geq 
    \sum_{s = 1}^K \{\alpha_s  \wedge (1 - \alpha_s)\}|\lambda_s| - \sum_{s = 1}^K \frac{p_s}{\bar\alpha}\eta(s) \{(2\alpha_s) \vee 1\}\enspace.
\end{align}

Observe that for any $(\bslambda, \bsgamma) \in \mathbb{R}^K \times  \mathbb{R}^K$ and for any $c \in \bbR$ the transformation
\begin{align*}
    \gamma_s \mapsto \gamma_s + \frac{p_s \alpha_s}{\bar\alpha}c,\quad\text{and}\quad \lambda_s \mapsto \lambda_s \quad s \in [K]\enspace,
\end{align*}
does not change the value of the objective function.
Take any minimizing sequence $(\bslambda^k, \bsgamma^k)$ of $H$. Due to the above observation we transform $(\bslambda^k, \bsgamma^k)$ to another minimizing sequence with the property
\begin{align}
    \label{eq:centered}
    \sum_{s = 1}^K \gamma_s^k = 0,\qquad\forall k \in \bbN\enspace.
\end{align}
By an abuse of notation we denote this transformed sequence by $(\bslambda^k, \bsgamma^k)$.
By definition of $(\bslambda^k, \bsgamma^k)$, for any $\epsilon > 0$ there exists $N \in \bbN$ such that
\begin{align*}
    H(\bslambda^k, \bsgamma^k) \leq H(\bszero, \bszero) + \epsilon, \qquad \forall k \geq N\enspace.
\end{align*}
Since,
\begin{align*}
    H(\bszero, \bszero) = \sum_{s=1}^K \frac{p_s}{2\bar\alpha} \left(\abs{1-2\eta(\bsX,s)} - 1\right)_+ = 0\enspace,
\end{align*}
it holds for all $k \geq  N$ that
\begin{align*}
    H(\bslambda^k, \bsgamma^k) \leq \epsilon, \qquad \forall k \geq N\enspace.
\end{align*}
Furthermore, since for all $k \in \bbN$ the property in Eq.~\eqref{eq:centered} holds, then using Eqs.~\eqref{eq:bound_gamma} and~\eqref{eq:bound_lambda} we obtain
\begin{align*}
    &\lVert \bsgamma^k \rVert_1 \leq 2(1 + \epsilon)\enspace,\\
    &\sum_{s = 1}^K \{\alpha_s  \wedge (1 - \alpha_s)\}|\lambda_s^k| \leq \epsilon + \sum_{s = 1}^K \frac{p_s}{\bar\alpha}\eta(s) \{(2\alpha_s) \vee 1\}\enspace.
\end{align*}
Thus for all $k \geq N$ the minimizing sequence $(\bslambda^k, \bsgamma^k)$ is bounded, extracting convergent sub-sequence and using the fact that $H$ is continuous we conclude that the global minimizer exists. 

\paragraph{Refined bound on $\bslambda$.}

Recall that the first-order optimality condition on $\bslambda^*$  (see \eqref{eq:KKT_lambda_star}) is given by 

\begin{align*}
        \alpha_s &= \Prob_{\bsX\mid S=s}\left( \left\lvert  \frac{p_s}{2\bar\alpha}(1-2\eta(\bsX,s)-\scalar{\bsgamma^*}{\bsone}) + \frac{\scalar{\bsgamma^*}{\bse_s}}{2\alpha_s} \right\rvert \geq \frac{p_s}{2\bar\alpha}(1-\scalar{\bsgamma^*}{\bsone}) + \scalar{\bslambda^*}{\bse_s} + \frac{\scalar{\bsgamma^*}{\bse_s}}{2\alpha_s}  \right), \forall s \in [K]\enspace. \\
\end{align*}

Since $\eta(x,s) \in [0, 1]$, then for any $\bsx \in \mathbb{R}^d$ it holds that
\begin{align*}
    -\frac{p_s}{\bar\alpha} - \frac{(\bsgamma_s^*)_-}{\alpha_s} \leq \left\lvert  \frac{p_s}{2\bar\alpha}(1-2\eta(\bsx,s)) + \frac{\bsgamma_s^*}{2\alpha_s} \right\rvert - \frac{p_s}{2\bar\alpha} - \frac{\bsgamma_s^*}{2\alpha_s} \leq - \frac{(\bsgamma_s^*)_-}{\alpha_s}\enspace.
\end{align*}

Therefore, if $\alpha_s$ is not in $\{0, 1\}$, we must have that
\begin{align*}
    -\frac{p_s}{\bar\alpha} \leq \bslambda_s^* + \frac{(\bsgamma_s^*)_-}{\alpha_s} \leq 0 \enspace,
\end{align*}
otherwise the considered probability is either equal to $0$ or to $1$.
In particular, it implies that
\begin{align*}
    \abs{\bslambda_s^*} \leq \frac{p_s}{\bar\alpha} \vee \frac{\abs{\bsgamma_s^*}}{\alpha_s}\enspace.
\end{align*}
Note that the same can be shown for $\hbslambda$ since
Eq.~\eqref{eq:empirical_reject_control} and Lemma~\ref{lem:pigh} imply
\begin{align*}
        \abs{\hat\Prob_{\bsX\mid S=s}\left( \left\lvert  \frac{p_s}{2\bar\alpha}(1-2\hat\eta(\bsX,s)-\hbsgamma_s) + \frac{\hbsgamma_s}{2\alpha_s} \right\rvert \geq \frac{p_s}{2\bar\alpha}(1-\hbsgamma_s) + \hbslambda_s + \frac{\bsgamma_s}{2\alpha_s}  \right) - \alpha_s} \leq \frac{2}{n_s}, \forall s \in [K]\enspace, \\
\end{align*}
and the assumption on $n_s$ guarantee that the empirical probability is strictly between $0$ and $1$.




\end{proof}

\section{Reduction to linear programming}
\label{app:optimization}

In this section we show that the minimization problem in Eq.~\eqref{eq:emp_min} can be reduced to a problem of linear programming.
Recall that our goal is to solve
\begin{align}
     \min_{(\bslambda, \bsgamma)} \ens{\scalar{\bslambda}{\bsalpha} + \sum_{s=1}^K \hat{\Exp}_{\bsX \mid S = s} 
     (\hat{G}(\bsX, s,
     \bslambda,
     \bsgamma ))_+}\enspace,
\end{align}
where 
\begin{align*}
    \hat{G}(\bsx, s, \bslambda, \bsgamma) =
    &\left\lvert  \frac{{p}_s}{2\bar\alpha}(1-2\hat\eta(\bsx,s)-\scalar{\bsgamma}{\bsone}) + \frac{\scalar{\bsgamma}{\bse_s}}{2\alpha_s} \right\rvert 
    - \frac{{p}_s}{2\bar\alpha}(1-\scalar{\bsgamma}{\bsone}) - \scalar{\bslambda}{\bse_s} - \frac{\scalar{\bsgamma}{\bse_s}}{2\alpha_s}\enspace.
\end{align*}
Similarly to the support vector machines, the reduction is achieved via the slack variables $\zeta_i$, $i = 1, \ldots, n$.
With these slack variables the above problem can be expressed as
\begin{equation}
\tag{\textbf{LP-Primal}}
\label{eq:lp}
    \begin{aligned}
        &\min_{(\bslambda, \bsgamma, \bszeta)}\scalar{\bslambda}{\bsalpha} + \sum_{s=1}^K\sum_{i \in \class{I}_s}\frac{\zeta_i}{n_s}\\
        &\text{s.t. }
        \begin{cases}
            \zeta_i \geq 0 &\forall i \in [n]\\
            0
            \leq \zeta_i + \scalar{\bslambda}{\bse_s} + \frac{{p}_s}{\bar\alpha}\hat\eta(\bsx_i,s)&\forall i \in \class{I}_s\forall s \in [K]\\
            0
            \leq \zeta_i + \scalar{\bsgamma}{\frac{1}{\alpha_s}\bse_s-\frac{{p}_s}{\bar\alpha}\bsone} + \scalar{\bslambda}{\bse_s}
            +\frac{{p}_s}{\bar\alpha}\left(1 - \hat\eta(\bsx_i,s) \right)&\forall i \in \class{I}_s\forall s \in [K]
        \end{cases}
    \end{aligned}
\end{equation}
To prove this result it is sufficient to observe that for all $x \in \bbR$ it holds that
\begin{align*}
    (x)_+ = \min_{\zeta \geq x, \zeta \geq 0} \zeta\enspace.
\end{align*}

Introduce the following notation
\begin{align*}
    &\bsc = \left(
    \underbrace{\sfrac{1}{n_1}, \ldots, \sfrac{1}{n_1}}_{\class{I}_1},
    \ldots
    \underbrace{\sfrac{1}{n_s}, \ldots, \sfrac{1}{n_s}}_{\class{I}_s},
    \ldots,
    \underbrace{\sfrac{1}{n_K}, \ldots, \sfrac{1}{n_K}}_{\class{I}_K},
    \alpha_1, \ldots, \alpha_K,
    0 \ldots, 0
    \right)\\
    &\bsy = (\bszeta^\top, \bslambda^\top, \bsgamma^\top)\\
    &\bsb = \left(
    \parent{\frac{p_1}{\bar\alpha}\heta(\bsx_i, s)}_{i \in \class{I}_1},
    \ldots,
    \parent{\frac{p_K}{\bar\alpha}\heta(\bsx_i, s)}_{i \in \class{I}_K},
    \parent{\frac{p_1}{\bar\alpha}(1  - \heta(\bsx_i, s))}_{i \in \class{I}_1},
    \ldots,
    \parent{\frac{p_K}{\bar\alpha}(1  - \heta(\bsx_i, s))}_{i \in \class{I}_K}\right)\\
    &\mathbf{A}
    =
    \renewcommand\arraystretch{1.3}
\mleft[
    \begin{array}{cccc|c|c}
        -\mathbf{I}_{n_1 \times n_1} & \mathbf{0}_{n_1 \times n_2} & \ldots & \mathbf{0}_{n_1 \times n_K} & -\mathbf{E}^1_{n_1 \times K} & \mathbf{0}_{n_2 \times K} \\
        \mathbf{0}_{n_2 \times n_1} & -\mathbf{I}_{n_2 \times n_2} & \ldots & \mathbf{0}_{n_2 \times n_K} & -\mathbf{E}^2_{n_2 \times K} & \mathbf{0}_{n_1 \times K}\\
        \vdots&\vdots&\ddots&\vdots&\vdots&\vdots\\
       \mathbf{0}_{n_K \times n_1} & \mathbf{0}_{n_K \times n_2} & \ldots & -\mathbf{I}_{n_K \times n_K} & -\mathbf{E}^K_{n_K \times K} & \mathbf{0}_{n_K \times K}\\
       \hline
        -\mathbf{I}_{n_1 \times n_1} & \mathbf{0}_{n_1 \times n_2} & \ldots & \mathbf{0}_{n_1 \times n_K} & -\mathbf{E}^1_{n_1 \times K} & \frac{p_1}{\bar\alpha}\mathbf{1}_{n_1 \times K} - \frac{1}{\alpha_1}\mathbf{E}^1_{n_1 \times K}\\
       \mathbf{0}_{n_2 \times n_1} & -\mathbf{I}_{n_2 \times n_2} & \ldots & \mathbf{0}_{n_2 \times n_K} & -\mathbf{E}^2_{n_2 \times K} &  \frac{p_2}{\bar\alpha}\mathbf{1}_{n_2 \times K} - \frac{1}{\alpha_2}\mathbf{E}^2_{n_2 \times K}\\
    \vdots&\vdots&\ddots&\vdots&\vdots&\vdots\\
       \mathbf{0}_{n_K \times n_1} & \mathbf{0}_{n_K \times n_2} & \ldots & -\mathbf{I}_{n_K \times n_K} & -\mathbf{E}^K_{n_K \times K} &  \frac{p_K}{\bar\alpha}\mathbf{1}_{n_K \times K} - \frac{1}{\alpha_K}\mathbf{E}^K_{n_K \times K}
    \end{array}
    \mright]
\end{align*}
where $\mathbf{E}^s_{n \times m}$ is a $n \times m$ matrix composed of zeros and ones, whose $s^{\text{th}}$ column is equal to $\bsone$ and all other elements are zero, $\mathbf{1}_{n \times m}$ is a matrix of ones of size $n \times m$.
Using the above notation, the problem in \eqref{eq:lp} can be written as
\begin{equation}
\tag{\textbf{LP-Primal-compacted}}
\label{eq:lp_compacted}
    \begin{aligned}
        &\min_{\bsy \in \bbR^{n + 2K}}\scalar{\bsc}{\bsy}\\
        &\text{s.t. }
        \begin{cases}
            \mathbf{A}\bsy \leq \bsb &\\
            y_i \geq 0 &i \in [n]
        \end{cases}
    \end{aligned}
\end{equation}
While the dimension of matrix $\mathbf{A}$  is $2n \times (n + 2K)$, this matrix has at most $4n + nK$ non-zero elements. This fact can be exploited if $n \gg K$, that it, the amount of \emph{unlabeled} data is large compared to the amount of groups.

\end{document}